%% file: main.tex
\documentclass[twoside,11pt]{article}

\usepackage{blindtext}
\usepackage{graphicx}
\usepackage{float}
\usepackage{caption}   
\usepackage{subcaption} 
\usepackage{bm}
\usepackage{enumitem}
\usepackage{adjustbox} 
\usepackage{mathtools}
\usepackage{amsmath}

\usepackage{algorithm}
\usepackage{algorithmic}
\usepackage[dvipsnames,table,svgnames]{xcolor}
\definecolor{deepblue}{RGB}{0,56,140} 
\definecolor{deepred}{RGB}{178,34,34}

\usepackage{tcolorbox}
\usepackage{xcolor}

\usepackage{wrapfig}
\tcbuselibrary{skins}        
\newcommand{\leadcomment}[1]{%
  \textbf{\textit{\textcolor{deepred}{#1}}}
}

\usepackage{multirow}
\usepackage{tabularx}
\usepackage{subcaption}

\usepackage[utf8]{inputenc} 
\usepackage[T1]{fontenc}    
\usepackage{url}            
\usepackage{booktabs}       
\usepackage{amsfonts}       
\usepackage{nicefrac}       
\usepackage{microtype}      
\usepackage{threeparttable}

\usepackage{siunitx}
\usepackage{tikz}
\usepackage{pifont} 

\definecolor{mydarkblue}{rgb}{0,0.08,0.45}
\definecolor{babyblue}{RGB}{137, 207, 240}
\definecolor{lightblue}{RGB}{173, 216, 230}
\definecolor{mydarkgreen}{RGB}{0, 139, 69}
\definecolor{MAEblue}{HTML}{C3DDEF}
\usepackage{subcaption}

\usepackage{titletoc}

\usepackage[abbrvbib, preprint]{jmlr2e}

\usepackage[capitalize,noabbrev]{cleveref}
\hypersetup{
	colorlinks=true,
	urlcolor=magenta,
	citecolor=mydarkblue,
}

\usepackage{lastpage}

\ShortHeadings{Transformers Are Born Biased}{}
\firstpageno{1}

\usepackage{amsmath}
\begin{document}

\title{
Transformers Are Born Biased: Structural Inductive Biases at Random Initialization and Their Practical Consequences
}

\author{\name Siquan Li\thanks{Equal contribution.} \email lisiquan@cuhk.edu.cn\\
      \addr The Chinese University of Hong Kong, Shenzhen
      \AND   
      \name Yao Tong$^{\textcolor{red}{*}}$ \email tongyao@u.nus.edu \\
      \addr National University of Singapore
      \AND
      \name Haonan Wang$^{\textcolor{red}{*}}$ \email haonan.wang@u.nus.edu\\
      \addr National University of Singapore
      \AND
      \name Tianyang Hu\thanks{Correspondence to Tianyang Hu.} \email hutianyang@cuhk.edu.cn \\
      \addr The Chinese University of Hong Kong, Shenzhen
}

\editor{My editor}

\maketitle

\begin{abstract}
 Transformers underpin modern large language models (LLMs) and are commonly assumed to be behaviorally unstructured at random initialization, with all meaningful preferences emerging only through large-scale training. 
 We challenge this assumption by showing that randomly initialized transformers already exhibit strong and systematic structural biases. 
 In particular, untrained models display extreme token preferences: across random input sequences, certain tokens are predicted with probabilities orders of magnitude larger.

 We provide a mechanistic explanation for this phenomenon by dissecting the transformer architecture at initialization. We show that extreme token preference arises from a contraction of token representations along a random seed-dependent direction. 
 This contraction is driven by two interacting forces: (i) asymmetric nonlinear activations in MLP sublayers induce global (inter-sequence) representation concentration, and (ii) self-attention further amplifies this effect through local (intra-sequence) aggregation. Together, these mechanisms align hidden representations along a direction determined solely by the random initialization, producing highly non-uniform next-token predictions.

 Beyond mechanistic insight, we demonstrate that these initialization-induced biases persist throughout training, forming a stable and intrinsic model identity. 
 Leveraging this property, we introduce \textit{SeedPrint}, a fingerprinting method that can reliably distinguish models that differ only in their random initialization, even after extensive training and under substantial distribution shift. 
 Finally, we identify a fundamental \textit{positional discrepancy} inherent to the attention mechanism's intra-sequence contraction that is causally linked to the \textit{attention-sink} phenomenon.
 This discovery provides a principled explanation for the emergence of sinks and offers a pathway for their control.

\end{abstract}


\section{Introduction}

Large Language Models (LLMs) have become the cornerstone of contemporary AI, driving advances in language understanding, code generation, reasoning, and beyond~\citep{brown2020language, achiam2023gpt, wei2022chain, touvron2023llama}. With access to astronomical amounts of data, it is often assumed that a model's capabilities are entirely shaped by training—that the vast pre-training and post-training processes alone define its personality, preference, and knowledge~\citep{kaplan2020scaling, hoffmann2022training, ouyang2022training}. Under this prevailing view, a model at random initialization is merely a blank slate, a white sheet of parameters awaiting instruction from data.

However, we reveal a surprising phenomenon: a randomly initialized transformer is not featureless. Even before seeing any data, it exhibits systematic biases in token preferences. 
As shown in Figure~\ref{fig:sub_left}, when performing next-token prediction on random input sequences, certain tokens are preferred by magnitudes larger than others. This counterintuitive observation challenges the blank-slate assumption and motivates a closer examination of the model's initialization regime.

Despite the celebrated success of transformers, the theoretical understanding of these architectures has lagged behind their empirical progress, and modern LLMs still function largely as black boxes. 
Existing studies primarily focus on scaling laws, optimization dynamics, or emergent behaviors \emph{after} training~\citep{kaplan2020scaling, liu2020understanding, noci2022signal, wei2022emergent}, while the initialization regime—where no learning has yet occurred—remains largely unexplored.

\begin{figure}[t]
    \centering
    \begin{subfigure}{0.48\textwidth}
        \centering
        \includegraphics[width=\textwidth]{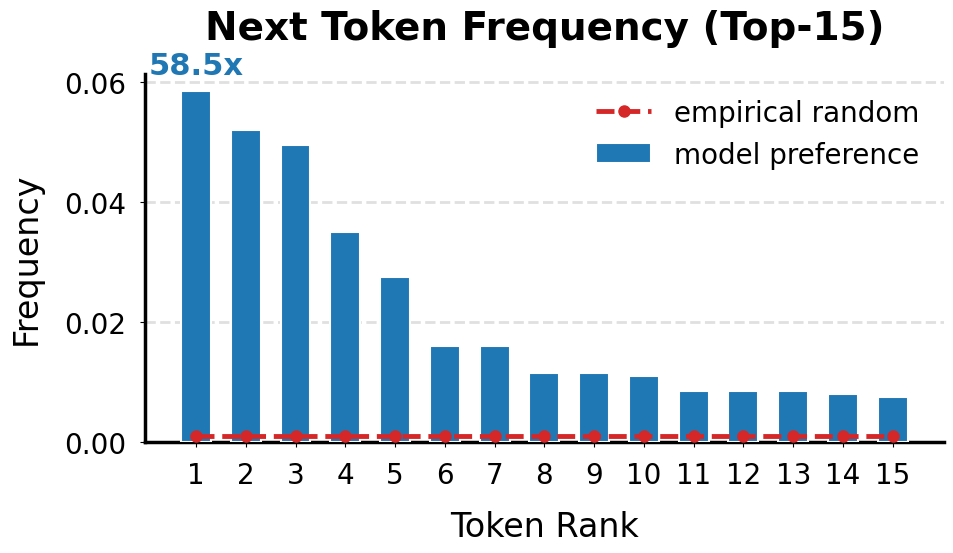}
        \caption{\ }
        \label{fig:sub_left} 
    \end{subfigure}
    \hfill
    \begin{subfigure}{0.48\textwidth}
        \centering
        \includegraphics[width=\textwidth]{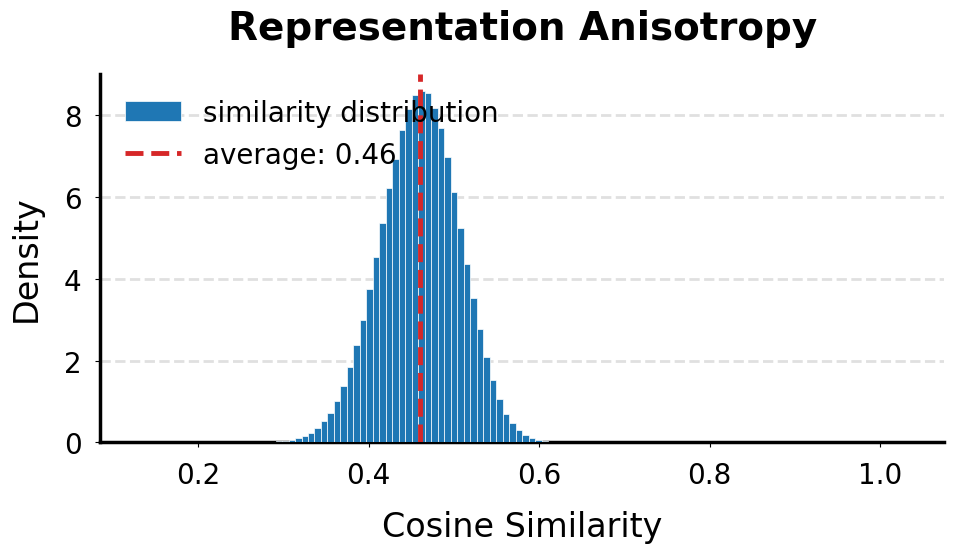}
        \caption{\ }
        \label{fig:sub_right} 
    \end{subfigure}
     \vspace{-6pt}
    \caption{\textbf{Initialized models are not blank states.}
    {(a)} When conducting next-token prediction on random sequences, randomly initialized transformer exhibits extreme biases where certain tokens are preferred by magnitudes larger than others. For reference, the red dashed line indicates the empirical top-ranked frequencies observed under uniform random sampling.
    {(b)} The token representation from random transformers are severely contracted towards a common direction, as indicated by the pairwise cosine similarity of the last-token representation among sequences.}
    \label{fig:main_figure} 
     \vspace{-9pt}
\end{figure}

To address this gap, we use random token sequences as probes to systematically study the initialization state of transformers. 
First, we show that the observed extreme token preferences are tied to random initialization: different random seeds give rise to distinct preference patterns.
Then, we analyze the internal dynamics from input to output and attribute the extreme next-token biases (Figure~\ref{fig:sub_right}) to an inherent contraction of token representations. 
Tracing this contraction through the transformer architecture, we identify two complementary forces that jointly drive the effect: 
\begin{itemize}
	\item 
    \textbf{Inter-sequence Concentration from MLP sublayers (\cref{sec:MLP_contract})}: 
    The asymmetric nonlinear activation (e.g., GELU) in MLP sublayers contracts representations of random input tokens toward a shared direction among different sequences.

	\item 
    \textbf{Intra-sequence Concentration from self-attention sublayers (\cref{sec:attn_contract,sec:MLP_attn_interplay})}
    The aggregation of value vectors in self-attention contracts token representations within the same sequence toward a shared direction, which in turn amplifies the existing MLP-induced inter-sequence contraction.
\end{itemize}

These results are conceptually striking: they demonstrate that the transformer architecture itself, even prior to training, possesses structural biases tied to random initialization.
In Section~\ref{sec:persistence}, we further reveal that these initial biases persists after training, forming a hidden and enduring model identity.

Finally, beyond mechanistic understanding, our findings also have practical implications. 
\begin{itemize}
    \item 
    \textbf{Biometric-prototype LLM fingerprinting} (\cref{sec:fingerprint}): Since the observed bias is \textit{model-distinct}, \textit{seed-specific}, and \textit{persistent through training}, it thus offers a natural foundation for birth-to-life model identification~\citep{galton1892finger} that does not rely on training artifacts. Utilizing the initial token biases, we propose a novel fingerprinting method that can even differentiate two LLMs trained with identical training pipelines but initialized with different random seeds.
    \item 
    \textbf{Attention-sink Mitigation} (\cref{sec:variance_and_sinks}): 
    We identify a fundamental but previously overlooked \textit{positional discrepancy} inherent to the attention mechanism's intra-sequence contraction. We demonstrate that this discrepancy is \textit{causally} linked to the widely observed attention-sink phenomenon~\citep{xiao2023efficient} in pretrained LLMs. 
    Uncovering this statistical origin provides a principled way to modulate sink strength through simple architectural adjustments. 
\end{itemize}

The rest of this paper is structured as follows. In Section~\ref{sec:preliminary}, we provide notation and formally introduce the structure of transformers. In Section~\ref{sec:ntp_preference}, we introduce the next-token extreme preference phenomenon in details. 
A mechanistic understanding of this phenomenon is provided in Section \ref{sec:rep_contract}, followed by empirical evidence in Section \ref{sec:persistence} demonstrating that such extreme token preference and the representation contraction persist during training.
In Section \ref{sec:practical}, we exploit this new-found understanding for practical benefits for practitioners and discuss potential impact in Section \ref{sec:discussion}.

\section{Preliminary}
\label{sec:preliminary}

\paragraph{Notations}
Bold uppercase letters $\bm{X}$ denote matrices, bold lowercase letters $\bm{x}$ denote vectors, and italic letters $x$ denote scalars. $\|\bm x\|$ denotes the $l_2$ norm of the vector $\bm x$. 
Unless otherwise specified, all vectors are column vectors, and all nonlinearities (e.g., GELU~\citep{Hendrycks2016}, SiLU~\citep{ramachandran2017searching}) are applied element-wise.
We use $\odot$ to denote element-wise (Hadamard) multiplication and $\oplus$ to denote residual addition.
$\bm{I_d}$ denotes $d$-dimensional identity matrix.

\subsection{Decoder-only Transformer Architecture}
The standard decoder-only transformer architecture~\citep{vaswani2017attention} has become the dominant backbone for modern LLMs~\citep{touvron2023llama, team2024qwen2, openai2025gpt5}.
Given a sequence of input token indices $\bm{i} = [i_1, \dots, i_T]^\top$ from a vocabulary $\mathcal{V}$, the transformer first maps these indices into a continuous vector space via an embedding layer 
$\bm{X}^{(0)} = \mathrm{Embed}(\bm{i}) \in \mathbb{R}^{T \times d},$ 
where $T$ denotes the sequence length and $d$ represents the hidden dimension.  
A standard transformer block then processes these representations through a sequence of interleaved components, typically consisting of self-attention mechanisms, position-wise Multi-Layer Perceptron (MLP) networks, and layer normalization modules.

\paragraph{self-attention}
For a single self-attention head, given hidden states $\bm{X} \in \mathbb{R}^{T \times d}$, the query, key and value matrices are computed as
\begin{align*}
\bm{Q} = \bm{X}\bm{W}_Q,\quad
\bm{K} = \bm{X}\bm{W}_K,\quad
\bm{V} = \bm{X}\bm{W}_V,
\end{align*}
where $\bm{W}_Q,\bm{W}_K,\bm{W}_V \in \mathbb{R}^{d \times d_k}$ are learnable projection matrices and $d_k$ denotes the head-dimension.
The attention mechanism computes a weighted sum of values based on the compatibility between queries and keys:
\begin{align*}
\mathrm{Attn}(\bm{X})
= \mathrm{Softmax}\!\left(\frac{\bm{Q}\bm{K}^\top}{\sqrt{d_k}}\right)\bm{V}.
\end{align*}
Multi-head attention (MHA) is constructed by executing $h$ such heads in parallel, concatenating their outputs, and applying a final linear projection $\bm{W}_O\in\mathbb{R}^{hd_k\times d}$.

\paragraph{Multi-Layer Perceptron} The MLP sublayer is applied to each token representation independently across the sequence.
Following the architecture popularized by \cite{radford2019language}, we consider a two-layer architecture with a non-linear activation $\phi$:
\begin{align*}
\mathrm{MLP}(\bm{X}) = \phi(\bm{X}\bm{W}_{\text{up}})\bm{W}_{\text{down}},
\end{align*}
where $\bm{W}_{\text{up}} \in \mathbb{R}^{d \times d_\text{MLP}}$ and
$\bm{W}_{\text{down}} \in \mathbb{R}^{d_\text{MLP} \times d}$ are learnable weight matrices,
and $d_\text{MLP}>d$ is the intermediate MLP hidden dimension. 
Popular choices for $\phi$ include GELU and SwiGLU. 
This sublayer expands the model capacity by projecting the hidden states into a higher-dimensional space before mapping them back to the original model dimension $d$.

\paragraph{Normalization}
Normalization is essential for stabilizing the training of both MHA and MLP sublayers.
The original transformer architecture utilized LayerNorm~\citep{ba2016layer}, which normalizes each token representation $\bm{x} \in \mathbb{R}^d$ by subtracting its mean $\bm{\mu}$ and dividing by its standard deviation $\sigma$, followed by a learned affine transformation:
\begin{align*}
\mathrm{LayerNorm}(\bm{x})
= \bm{\gamma} \odot \frac{\bm{x} - \bm{\mu}}{\sqrt{\sigma^2 + \epsilon}} + \bm{\beta},
\end{align*}
where $\bm{\gamma}, \bm{\beta} \in \mathbb{R}^d$ are learnable scale and shift parameters, and $\epsilon$ is a small constant for numerical stability.
Modern LLMs such as LLaMA-4~\citep{llama4_blog} and Qwen3~\citep{yang2025qwen3} typically adopt Root Mean Square Layer Normalization (RMSNorm)~\citep{zhang2019root}, which omits mean-centering (and often the bias term) and only rescales by the root mean square of the activations:
\begin{align*}
\mathrm{RMSNorm}(\bm{x})
= \bm{\gamma} \odot \frac{\bm{x}}{\sqrt{\frac{1}{d}\sum_{i=1}^d x_i^2 + \epsilon}}.
\end{align*}
Unless otherwise specified, Norm$(\cdot)$ in our equations denotes LayerNorm.

\paragraph{Block Structure and Output Layer}
Each transformer block comprises an MHA sublayer followed by a position-wise MLP sublayer. We employ the pre-norm configuration with residual connections~\citep{he2016deep}:
\begin{align*}
\bm{H}^{(l)} &= \bm{X}^{(l-1)} \oplus \mathrm{MHA}\big(\mathrm{Norm}(\bm{X}^{(l-1)})\big), \\
\bm{X}^{(l)} &= \bm{H}^{(l)} \oplus \mathrm{MLP}\big(\mathrm{Norm}(\bm{H}^{(l)})\big),
\end{align*}
for layers $l = 1,\dots,L$.
After $L$ decoder blocks, the final hidden states $\bm{X}^{(L)} \in \mathbb{R}^{T \times d}$ undergo a final normalization and an output projection (the language-model head) to obtain logits
$\bm{L} \in \mathbb{R}^{T \times |\mathcal{V}|}$ over the vocabulary:
\begin{align}
\label{eq:next_token_logits}
\bm{L} = \mathrm{Norm}_{\text{final}}(\bm{X}^{(L)}) \bm{W}_U,
\end{align}
where $\bm{W}_U \in \mathbb{R}^{d \times |\mathcal{V}|}$ is the unembedding matrix.
We follow the common practice of weight tying~\citep{press2016using} and set $\bm{W}_U = \bm{W}_E^\top$.

\paragraph{Autoregressive Language Modeling}

The transformer backbone described above parameterizes an autoregressive language model.
Given a sequence of tokens $x_1, x_2, \dots, x_T$, the model assumes the usual factorization
\begin{align*}
P(x_1, x_2, \dots, x_T)
= \prod_{t=1}^{T} P(x_t \mid x_1, \dots, x_{t-1}),
\end{align*}
where each conditional distribution $P(x_t \mid x_{<t})$ is obtained by applying a softmax to the logits at position $t$ in Equation~\eqref{eq:next_token_logits}.
In the remainder of the paper, we study how the architectural choices above, together with standard initialization, induce systematic preferences over next-token probabilities even before training.

\subsection{Current LLM Research Frontline}

Modern LLM research predominantly focuses on the capabilities and behaviors that emerge \textit{after} substantial training.
This includes scaling laws and pretraining efficiency~\citep{hoffmann2022training, sun2024massive,kaplan2020scaling}, post-training alignment and reasoning~\citep{lin2022truthfulqa, wang2022self, wang2025prefix, wang2025understanding, ren2024learning, guo2025deepseek}, and model interpretability~\citep{elhage2021mathematical, allen2023physics, dai2022knowledge, chen2024exact, rai2024practical}. Under this prevailing paradigm, a model at random initialization is often viewed as a ``blank slate''—a featureless collection of parameters awaiting instruction from data~\citep{bommasani2021opportunities, mueller2022coloring}.

In contrast, we propose a shift in perspective from ``what the model learns'' to ``what the model is born with''. Our work investigates the fundamental structural inductive biases present at the very beginning: the initialization regime. 
We argue that the transformer architecture is not a neutral container for data, but rather a structured system with innate computational tendencies. 
This fundamental perspective contributes to the community by demonstrating that initialization-born structures are the mechanistic root of certain model behaviors.

Specifically, our discovery provides a physical basis for SeedPrint, a LLM fingerprinting method that enables reliable ``birth-to-life'' identification by leveraging seed-specific biases that persist throughout the model's lifecycle. 
Furthermore, we demystify the widespread attention-sink phenomenon, revealing it to be a statistical byproduct of the architecture's initial variance discrepancy rather than a learned strategy. By isolating these effects, we offer a direct pathway for controlling such behaviors through principled architectural adjustments rather than intensive data-driven tuning.

\indent In the following sections, we move beyond these high-level observations to provide a rigorous mechanistic account of these biases, beginning with the empirical characterization of the extreme token preference phenomenon.

\section{Extreme Token Preference at Initialization}\label{sec:ntp_preference}

We use random sequences as a controlled probe to investigate transformers at initialization. 
Studying how a randomly initialized transformer processes random input offers a principled way to reveal the innate computational biases embedded in its architecture and initialization.
Specifically, we conduct our exploratory experiments on randomly initialized nano-scale LLMs, including RoPE-enhanced GPT-2~\citep{radford2019language} and LLaMA-2~\citep{touvron2023llama}. 
To ensure architectural comparability, both models adopt a consistent GPT-style parameter initialization\footnote{The specific details of this initialization are provided in Appendix~\ref{sec:appendix_init_details}.}. 
Specifically, we generate $N=2,000$ input sequences, where each token is sampled uniformly and independently from the vocabulary $\mathcal{V}$ ($|\mathcal{V}| \approx 50,000$). 
For each sequence, we perform a forward pass through the randomly initialized model and record the token predicted at the final position (i.e., the token with the maximum logit).
Below we present our findings. 

\begin{tcolorbox}[
    enhanced,
    colback=LightGoldenrodYellow!40,
    colframe=black!20,
    arc=2mm, 
    boxrule=0.5pt, 
    boxsep=0mm,
    before skip=5mm,
    after skip=5mm
]
\textbf{Finding 1:} \textit{Across different random sequences, randomly initialized transformers have abnormally large tendency to predict the same token. }
\end{tcolorbox}

Figure~\ref{fig:main_figure}(a) illustrates the frequency distribution of the top-ranked predicted tokens derived from this process.
If the initialized model were truly a blank slate or featureless, the predicted tokens should be distributed roughly uniformly across the vocabulary.
This baseline is depicted by the red dashed line, which represents the empirical top-ranked frequencies observed when sampling $N$ tokens completely at random from the vocabulary $\mathcal{V}$.
However, the actual model behavior (blue bars) deviates drastically from this baseline. 
Instead of a uniform spread, a tiny subset of tokens dominates the predictions. The most preferred token appears 58.5 times more frequent than the random baseline.

To demonstrate the robustness of this random token prediction bias, we evaluate models across three distinct architectures: (a) \textit{RoPE-enhanced Nano GPT-2}, (b) \textit{Nano LLaMA-2}, and (c) larger \textit{RoPE-enhanced 1.2B GPT-2}. 
For each model configuration, we conduct independent trials varying the input sequence length and random initialization seed. 
Each trial involves generating $N=10,000$ random input sequences\footnote{To ensure the comparability of token identifiers across different random seeds within the same architecture, we maintain \textit{fixed weights} for the embedding and language modeling head layers, isolating the source of variation to the internal transformer blocks.}. 
For quantitative assessment, we analyze the statistical significance of the observed token preference by calculating the $p$-values under the null hypothesis of a uniform distribution\footnote{We applied \textit{Bonferroni correction} to strictly account for multiple hypothesis testing across the vocabulary size. Specific details of experimental setup are provided in Appendix~\ref{sec:appendix_consistent_details}.}. 
The results are summarized in Table~\ref{tab:integrated_bias_stats}, which explicitly reports the {most frequent token (top-1 ID)}, its frequency, and statistical significance. We observe three prominent characteristics regarding the phenomenon:

\begin{tcolorbox}[
    enhanced,
    colback=LightGoldenrodYellow!40,
    colframe=black!60,
    arc=2mm, 
    boxrule=0.5pt, 
    boxsep=0mm,
    before skip=5mm,
    after skip=5mm
]
\textbf{Finding 1.1:} \textit{The extreme token preference remains highly consistent and robust across diverse model architectures, scales, and input lengths. }
\end{tcolorbox} 

Regardless of whether a RoPE-enhanced GPT-2 or a LLaMA-2 architecture is used, the model consistently exhibits a non-random preference for specific tokens, confirming that the bias is a structural property rather than a stochastic artifact.
  
\begin{tcolorbox}[
    enhanced,
    colback=LightGoldenrodYellow!40,
    colframe=black!60,
    arc=2mm, 
    boxrule=0.5pt, 
    boxsep=0mm,
    before skip=5mm,
    after skip=5mm
]
\textbf{Finding 1.2:} \textit{The larger the model, the longer the sequence, the more extreme the token preference. }
\end{tcolorbox} 

The intensity of the bias scales positively with both model scale and input sequence length. For instance, in the 1.2B GPT-2 model (part (c)) under seed 43, the top-1 frequency increases from $5.20\%$ at a sequence length of 64 to a striking $10.20\%$ at a length of 1024. Comparing part (a) with part (c), the 1.2B model consistently exhibits higher bias intensities than its nano counterpart at the same input length, suggesting that larger scales and longer sequences further intensify the next-token prediction bias.

\begin{tcolorbox}[
    enhanced,
    colback=LightGoldenrodYellow!40,
    colframe=black!60,
    arc=2mm, 
    boxrule=0.5pt, 
    boxsep=0mm,
    before skip=5mm,
    after skip=5mm
]
\textbf{Finding 1.3:} \textit{The extreme token preference is tied to the random initialization of transformers. }
\end{tcolorbox}

Even for models with the identical architecture and shared embedding weights, varying the initialization seeds consistently leads to a convergence on different top-1 IDs. For instance, in the 1.2B GPT-2 model, while seed 42 consistently favors token 23062 across all sequence lengths, seed 43 shifts this preference to token 20694. This evidence confirms that while the phenomenon of prediction bias is a universal structural property, its specific target is uniquely determined by the random initial weights.

\begin{table}[t]
\centering
\begin{threeparttable}
\caption{Token prediction bias statistics across different models. We report the identity of the most predicted token (top-1 ID), its frequency in percentage (top-1 Freq. \%), and its statistical significance ($p$-value). The results are grouped by model architecture: (a) RoPE-enhanced Nano GPT-2, (b) Nano LLaMA-2, and (c) RoPE-enhanced 1.2B GPT-2.}
\label{tab:integrated_bias_stats}
\small
\renewcommand{\arraystretch}{0.88} 
\setlength{\extrarowheight}{0pt}   
\setlength{\tabcolsep}{20pt}
\begin{tabular}{
    S[table-format=2.0]
    S[table-format=4.0]
    S[table-format=5.0]
    S[table-format=2.2, table-space-text-post=\%]
    S[table-format=1.2e-3, table-space-text-post=\tnote{$\dagger$}]
}
\toprule
\multicolumn{1}{c}{Seed} & 
\multicolumn{1}{c}{SeqLen} & 
\multicolumn{1}{c}{Top-1 ID} & 
\multicolumn{1}{c}{Top-1 Freq. (\%)} & 
\multicolumn{1}{c}{$p$-value} \\

\midrule
\multicolumn{5}{l}{\textbf{(a) RoPE-enhanced Nano GPT-2} (12L, 12H, 768d)} \\
\midrule
42 & 64   & 6336  & 2.60\% & 4.55e-137 \\
42 & 256  & 30425 & 4.85\% & 6.59e-285 \\
42 & 1024 & 666   & 5.20\% & 6.75e-309 \\
\addlinespace
43 & 64   & 7328  & 3.00\% & 2.20e-162 \\
43 & 256  & 7328  & 5.25\% & 2.42e-312 \\ 
43 & 1024 & 7328  & 6.75\% & 0.00\tnote{$\dagger$} \\ 

\midrule
\multicolumn{5}{l}{\textbf{(b) Nano LLaMA-2} (12L, 12H, 768d)} \\
\midrule
42 & 64   & 18219 & 0.06\% & 3.16e-2 \\
42 & 256  & 19647 & 0.06\% & 3.16e-2 \\
42 & 1024 & 11048 & 0.07\% & 1.40e-3 \\
\addlinespace
43 & 64   & 6479  & 0.09\% & 1.90e-6 \\
43 & 256  & 17743 & 0.06\% & 3.16e-2 \\
43 & 1024 & 21436 & 0.07\% & 1.40e-3 \\

\midrule
\multicolumn{5}{l}{\textbf{(c) RoPE-enhanced 1.2B GPT-2} (24L, 32H, 2048d)} \\
\midrule
42 & 64   & 23062 & 3.95\% & 2.57e-224 \\
42 & 256  & 23062 & 6.20\% & 0.00\tnote{$\dagger$} \\
42 & 1024 & 23062 & 6.35\% & 0.00\tnote{$\dagger$} \\
\addlinespace
43 & 64   & 20694 & 5.20\% & 6.75e-309 \\
43 & 256  & 20694 & 7.45\% & 0.00\tnote{$\dagger$} \\
43 & 1024 & 20694 & 10.20\% & 0.00\tnote{$\dagger$} \\
\bottomrule
\end{tabular}

\begin{tablenotes}
    \footnotesize
    \item[$\dagger$] The $p$-value is reported as 0.00 due to numerical underflow in floating-point precision.
\end{tablenotes}
\end{threeparttable}
\end{table}

\section{Representation Contraction at Initialization}\label{sec:rep_contract}

In this section, we elucidate the underlying mechanism leading to the observed extreme token preference. 
Tracing back the next-token prediction process, the logits are determined by the product of the last token representation and the unembedding matrix as in Equation \eqref{eq:next_token_logits}.
Since the unembedding matrix (row-wise) is randomly initialized with independent and identically distributed Gaussian vectors that are approximately orthogonal and of similar norms, a reasonable suspect responsible for the observed extreme preference is the \textit{last token representations}.

To uncover the abnormalities in the last token representations, we conduct an empirical study using a RoPE-enhanced Nano GPT-2 model. Specifically, we feed the model with 2,000 sequences, each consisting of 1,024 randomly sampled tokens. 
For each input sequence, we extract the representations of the last token from the output of every transformer block. 
Then, layer by layer, we compute the \textit{pairwise cosine similarities} among the last-token representations from different sequences. 
This metric measures the representational similarity of the last token across sequences.
The results are illustrated in Figure~\ref{fig:pairwise_cosine_similarity}, which characterizes the evolution of the last token representations through the network layers. 
The line chart in Figure~\ref{fig:pairwise_cosine_similarity} reveals that the last-token representations of different random input sequences exhibit substantial cosine similarity, indicating a clear directional contraction in the representation space across sequences. We term this phenomenon, in which the last-token representations of different sequences become increasingly similar, \textbf{\textit{inter-sequence contraction}}.
As the model deepens, this contraction becomes more pronounced, as evidenced by the increasing trend in the average pairwise cosine similarity of the last token.

\begin{tcolorbox}[
    enhanced,
    colback=LightGoldenrodYellow!40,
    colframe=black!20,
    arc=2mm, 
    boxrule=0.5pt, 
    boxsep=0mm,
    before skip=5mm,
    after skip=5mm
]
\label{finding:contraction}
\textbf{Finding 2.1:} \textit{The last token representations from a random transformer exhibit severe inter-sequence contraction.}
\end{tcolorbox}

\begin{figure}[t]
    \centering
    \includegraphics[width=0.8\textwidth]{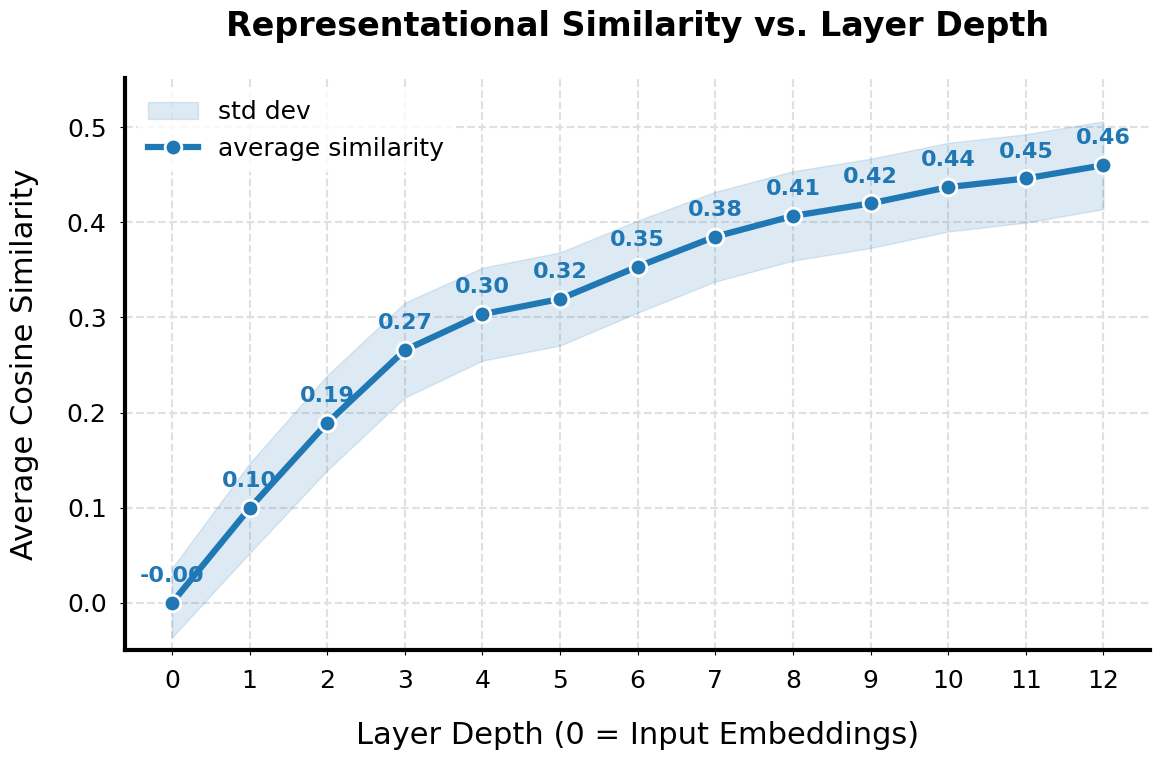}
    \caption{Pairwise cosine similarity of last-token representations between different sequences and its evolution with increasing transformer blocks.}
    \label{fig:pairwise_cosine_similarity}
\end{figure}

To further evaluate the relationship between the contracted direction and the model's token prediction preference, we revisit the logit formulation in Equation \eqref{eq:next_token_logits}. 
First, we define the \textit{representation contraction direction} as the average of the normalized final-layer representations at the last token position (Norm$_{\text{final}}(\bm{X}_{\mathrm{last}}^{(L)})$). 
Then, we treat this direction as the final hidden states and project it via $\bm{W}_U$ to obtain the corresponding logits. 
We are interested in whether the token with the maximum logit (most aligned with the representation contraction direction) is identical to the observed most frequently predicted next token presented in Section \ref{sec:ntp_preference}.

\begin{tcolorbox}[
    enhanced,
    colback=LightGoldenrodYellow!40,
    colframe=black!20,
    arc=2mm, 
    boxrule=0.5pt, 
    boxsep=0mm,
    before skip=5mm,
    after skip=5mm
]
\textbf{Finding 2.2:} \textit{The inter-sequence representation contraction direction aligns with the favorite token.}
\label{finding:contraction_direction}
\end{tcolorbox}

Specifically, for each of the 100 randomly initialized models, we feed 2,000 sequences of length 1,024 consisting of random tokens. 
The results demonstrate a $76\%$ overlap. 
This indicates a significant alignment between the contracted direction in the representation space and the embedding of the most likely token.

Combining Findings 2.1 and 2.2, we can draw the following conclusion.
\begin{tcolorbox}[
    enhanced,
    colback=LightGoldenrodYellow!40,
    colframe=black!20,
    arc=2mm, 
    boxrule=0.5pt, 
    boxsep=0mm,
    before skip=5mm,
    after skip=5mm
]
\textbf{Finding 2:} \textit{The extreme next-token prediction preference stems from the inter-sequence representation contraction.}
\end{tcolorbox}

In the next section, we dive deeper into the transformer model architecture to expose the root cause and underlying mechanism of the representation contraction.

\subsection{Architectural Ablation \label{sec:arch_aba}}
A transformer consists of an interweaving of self-attention and MLP modules. 
To attribute the source of the representation contraction, we analyze the self-attention and MLP modules in isolation. 
For simplicity, we omit the embedding layer and directly feed random Gaussian vectors into the transformer to simulate the embeddings of random sequences.
Crucially, these input vectors are sampled with a mean and standard deviation strictly identical to the initialization strategy of the omitted embedding layer.
This simplification prevents token overlap from random sampling that would complicate the similarity analysis.
Unless explicitly stated otherwise, all subsequent experiments utilize this Gaussian input setting.

We evaluate three architectural configurations: 
\begin{itemize}
    \item full transformer model (self-attention + MLP); 
    \item model with only self-attention blocks (self-attention-only); 
    \item model with only MLP blocks (MLP-only). 
\end{itemize}
Each variant consists of 12 layers and is probed using random sequences of Gaussian vectors to compare their internal processing dynamics.
First, we characterize the \textit{next-token preference} for each variant following the procedure detailed in Section \ref{sec:ntp_preference}.
The resulting distributions are shown in Figure \ref{fig:next_token_bias_orthogonal}. 
Second, we quantify the \textit{inter-sequence representation contraction} by computing the pairwise cosine similarities between the last-layer last-token representations (extracted from the final LayerNorm) across all input sequences. The average value and statistical significance are reported in Table \ref{tab:cosine_sim_results}.

\begin{figure}[t]
    \centering
    \includegraphics[width=0.82\textwidth]{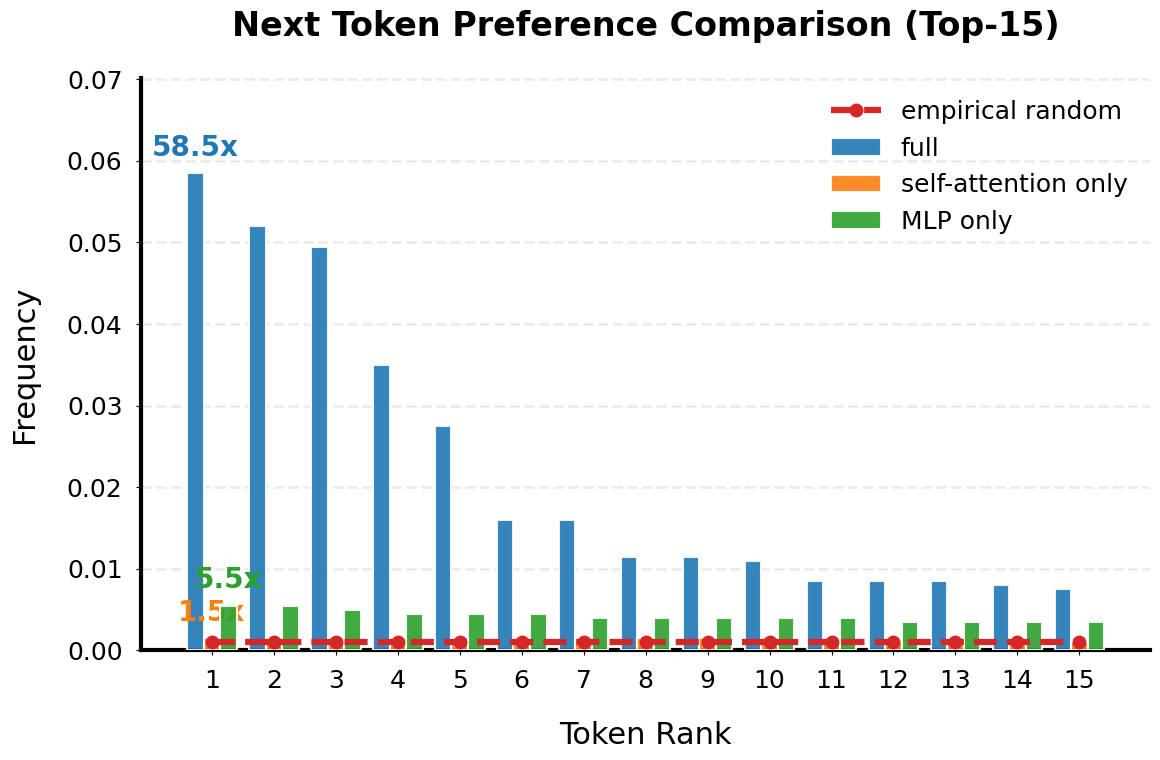}
     \vspace{-5pt}
    \caption{Next-token preference induced by self-attention and MLP modules separately and combined. The self-attention-only model (orange) is flat, aligning with the empirical random baseline, while the MLP-only (green) and full (blue) models show strong preference.}
    \label{fig:next_token_bias_orthogonal}
     \vspace{-5pt}
\end{figure}

\begin{tcolorbox}[
    enhanced,
    colback=LightGoldenrodYellow!40,
    colframe=black!20,
    arc=2mm, 
    boxrule=0.5pt, 
    boxsep=0mm,
    before skip=5mm,
    after skip=5mm
]
\textbf{Finding 3.1:} \textit{The extreme token preference and the inter-sequence representation contraction are due to MLP modules.}
\end{tcolorbox}

\begin{tcolorbox}[
    enhanced,
    colback=LightGoldenrodYellow!40,
    colframe=black!20,
    arc=2mm, 
    boxrule=0.5pt, 
    boxsep=0mm,
    before skip=5mm,
    after skip=5mm
]
\textbf{Finding 3.2:} \textit{While attention modules alone do not introduce extreme token preference nor inter-sequence representation contraction, their pairing with MLP modules boosts the severity of the abnormalities.}
\end{tcolorbox}

As can be seen in Figure \ref{fig:next_token_bias_orthogonal}, the self-attention-only model (orange) exhibits a flat histogram that closely aligns with the empirical random baseline, indicating no clear token preference.
The MLP-only model (green) produces a clear spiky preference for specific outlier tokens.
The full model (blue) produces the most extreme preference outliers, far exceeding the other two.
Similar interpretations can also be made from Table~\ref{tab:cosine_sim_results}, where the self-attention-only model maintains near-orthogonal representations, whereas the MLP-only model exhibits significant contraction, with the full model displaying the most severe contraction.

\begin{table}[t]
\centering
{
\caption{\small
    Architectural ablation for extreme next-token preference by comparing the average pairwise cosine similarity of the last-token representations across different sequences and their statistical significance.
    $\dagger$ The $p$-value is reported as 0.00 due to numerical underflow in floating-point precision. Refer to \cref{sec:appendix_consistent_details} for the detailed calculation protocol.}
    \small \label{tab:cosine_sim_results}
\setlength{\tabcolsep}{44pt}
\begin{tabular}{lrr}
\toprule
Model Mode & Avg. Sim & $p$-value \\
\midrule
full Model & $0.46$ & $0.00^{\dagger}$ \\
self-attention-only & $1.11 \times 10^{-4}$ & $0.16$ \\
MLP-only & $0.25$ & $0.00^{\dagger}$ \\
\bottomrule
\end{tabular}
}
\end{table}

These results suggest that the MLP module is the primary driver of this initial representation contraction. 
Furthermore, the interaction between the self-attention and MLP modules creates a significantly more pronounced contraction effect, which in turn causes the strong next-token preference seen in the full model.
In the following sections, we conduct deeper investigations to uncover the underlying mechanism.

\subsection{MLP-Induced Contraction of Inter-Sequence Representations} \label{sec:MLP_contract}
Our ablation study on the model architecture reveals that MLP blocks alone are sufficient to induce severe token preference bias and representation contraction. In this section, we investigate how such bias originates and propagates through the network.

We first zoom in on an MLP-only architecture and analyze its layer-wise behavior to pinpoint which components within the MLP are responsible for this contraction. Following the same pipeline as before, we generate $N=2,000$ distinct input sequences sampled from Gaussian distribution.
These sequences are processed through a stack of $L=12$ randomly initialized MLP blocks, following the standard GPT-style initialization (details in Appendix~\ref{sec:appendix_init_details}).
After each block, we compute the average and standard deviation of the pairwise cosine similarity across all $N$ last-token representations to track the evolution of inter-sequence contraction. Figure~\ref{fig:MLP_inter_collapse} shows that the average of pairwise cosine similarity steadily increases as the representations pass through successive MLP blocks.

\begin{figure}[h!]
    \centering
    \includegraphics[width=0.8\textwidth]{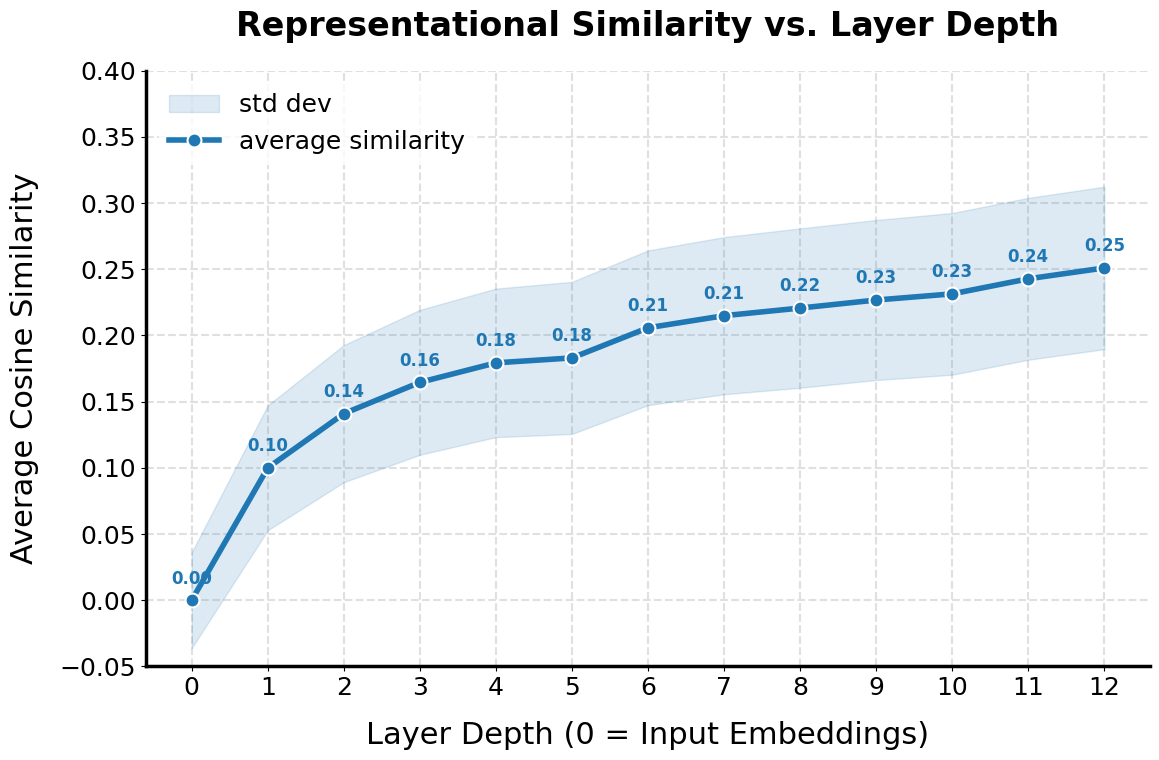} 
     \vspace{-8pt}
    \caption{
        Average and standard deviation of the pairwise cosine similarity between the last-token representations of different sequences, measured after each successive MLP block. 
    }
    \label{fig:MLP_inter_collapse}
     \vspace{-10pt}
\end{figure}

To understand this phenomenon, we examine the internal structure of the MLP block.
In the MLP-only architecture, each MLP block consists of the pre-MLP LayerNorm, the two-layer perceptron with GELU activation, and the subsequent residual connection. Specifically,
\begin{align*}
\bm{Y} &= \mathrm{MLP}\big(\mathrm{Norm}(\bm{X})\big) \oplus \bm{X},\\
\mathrm{MLP}(\bm{X}) &= \mathrm{GELU}\big(\bm{X}\bm{W}_\text{up}\big)\,\bm{W}_\text{down}.
\end{align*}

Since random initialized linear projections are unlikely to introduce significant representation contraction, we hypothesize that the contraction is caused by the \textit{asymmetric activation function} in the MLP block.

\begin{tcolorbox}[
    enhanced,
    colback=LightGoldenrodYellow!40,
    colframe=black!20,
    arc=2mm, 
    boxrule=0.5pt, 
    boxsep=0mm,
    before skip=5mm,
    after skip=5mm
]
\textbf{Finding 4:} \textit{Asymmetric nonlinear activation in MLP can cause inter-sequence representation contraction.}
\end{tcolorbox}

To isolate the role of nonlinear activation from the effects of normalization and residual connections, we first consider a simplified MLP setting with both components removed.
\begin{definition}[MLP$_0$ block]
\label{def:relu_MLP}
The simplified MLP block, denoted as MLP$_0$, processes the input through a single feedforward pass without residual paths or normalization. The output of the MLP$_0$ block is 
$\bm{Y} = \phi\big(\bm{X}\bm{W}_\text{up}\big)\,\bm{W}_\text{down}$ where $\phi$ denotes the activation function.
\end{definition}
For convenience, we mainly consider two choices for $\phi$:
(1) \textit{ReLU}, serving as an asymmetric nonlinear baseline; and (2) \textit{tanh}, serving as a symmetric nonlinear baseline.
For each of the these MLP$_0$ block variants, similar to the standard model, we calculate the pairwise cosine similarities between the last-token representations and track its progression as we include more blocks\footnote{Crucially, to ensure a controlled comparison, both baselines are initialized with identical random weights. }.
Figure~\ref{fig:MLP_relu_vs_tanh)_inter_collapse} summarizes the results. For comparison, the standard MLP block is also included.

\begin{figure}[t]
    \centering
    \includegraphics[width=0.8\textwidth]{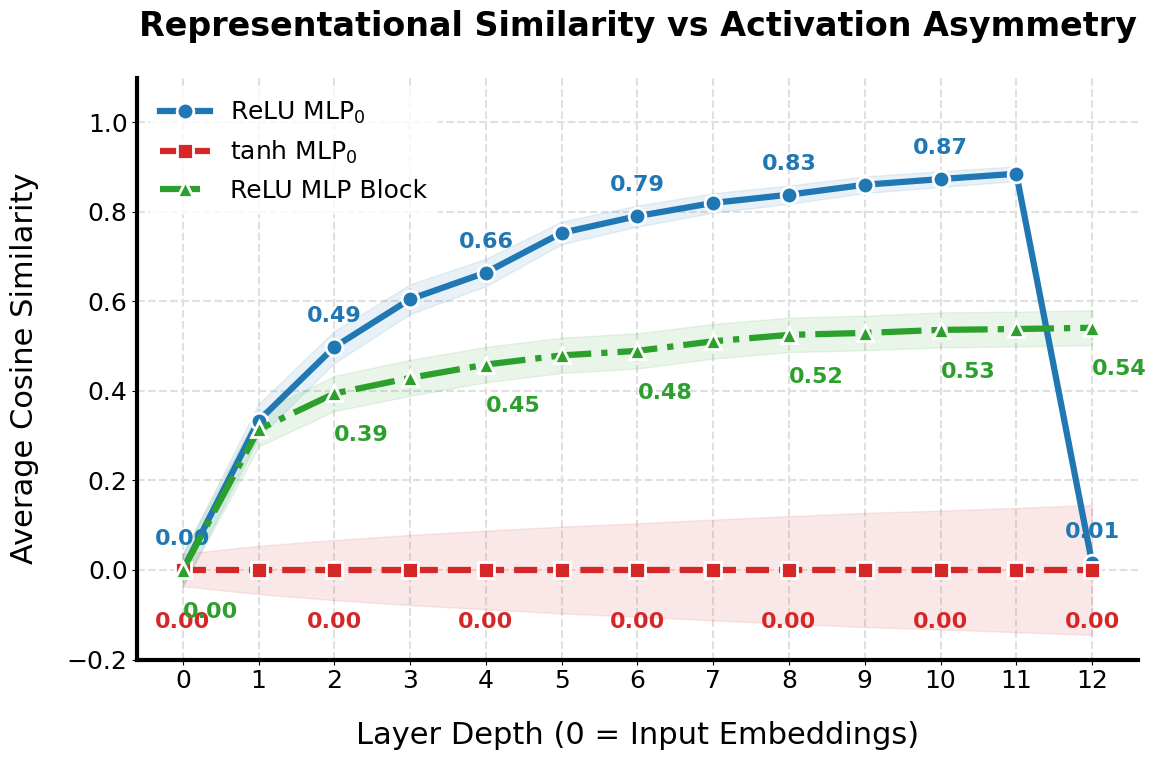}
    \vspace{-8pt}
    \caption{
        Average pairwise cosine similarity analysis. 
        We compare the asymmetric ReLU (blue) against the symmetric tanh (red) in a simplified setting without LayerNorm or residuals to isolate the effect of activation symmetry. 
        The full ReLU based MLP block with residual connections and LayerNorm (green) is included for reference, demonstrating that the contraction phenomenon persists in the standard architecture but is moderated by the residual mechanism.
    }
    \label{fig:MLP(relu_vs_tanh)_inter_collapse}
    \vspace{-1pt}
\end{figure}

A stark contrast is observed in Figure~\ref{fig:MLP(relu_vs_tanh)_inter_collapse}. 
While the symmetric tanh baseline (red dashed line) maintains an average pairwise cosine similarity of approximately zero, the asymmetric ReLU baseline (blue line) exhibits a rapid and significant increase in similarity\footnote{The drop in deeper layers for the ReLU model may be attributed to the representations exponentially collapsing towards the zero vector due to vanishing variance, a known phenomenon in deep networks that lack LayerNorm and residual connections~\citep{pmlr-v9-glorot10a, he2015delving}}.
The standard MLP-only (green line) model, which incorporates ReLU activation alongside residual connections and LayerNorm, also demonstrates a steady increase in similarity, albeit at a more moderate rate compared to the standalone MLP$_0$.
This difference is attributable to the residual connections, which preserve input diversity and dampen the rapid contraction observed in the residual-free setting.
Together, these empirical results support our hypothesis that activation asymmetry is the primary driver of inter-sequence representation contraction.

To ground these observations, we provide theoretical evidence demonstrating that asymmetric activation functions such as ReLU provably induce extreme inter-sequence representation contraction in the MLP$_0$ setting.

\begin{proposition}
\label{lem:asymmetric_collapse}
Let $\bm X_1$ and $\bm X_2$ be two independent Gaussian vectors with mean zero and covariance $\sigma^2\bm{I_d}$. 
Denote $\bm Z_l^{\mathrm{ReLU}}(\bm X)=\mathrm{MLP}_0^{(l)}\circ\cdots\circ\mathrm{MLP}_0^{(1)}(\bm X)$ as the output after $l$ layers of independent MLP mappings with ReLU activation.
Then, the expected cosine similarity, $\bar{\rho}_l^{\mathrm{ReLU}}=\mathbb{E}\left(\rho(\bm Z_l^{\mathrm{ReLU}}(\bm X_1), \bm Z_l^{\mathrm{ReLU}}(\bm X_2))\right)$, satisfies:
\begin{itemize}
    \item 
    $\bar{\rho}_1^{\mathrm{ReLU}}=1/\pi$;
    \item 
    $\bar{\rho}_l^{\mathrm{ReLU}}$ is monotonically increasing with $l$;
    \item 
    $\lim_{l\to\infty}\bar{\rho}_l^{\mathrm{ReLU}}=1$.
\end{itemize}
In contrast, if ReLU is substituted by tanh, $\mathbb{E}(\bar\rho_l^{\mathrm{tanh}})=0$ for any $l$.
\end{proposition}

Proposition~\ref{lem:asymmetric_collapse} establishes that asymmetric activations such as ReLU inherently drive inter-sequence similarity toward total collapse ($\rho=1$), whereas symmetric activations such as tanh maintain orthogonality. 
The formal proof, where the mathematical role of symmetry is central, is provided in~\cref{sec:proof_prop2}.

While the analysis in Proposition \ref{lem:asymmetric_collapse} establishes the fundamental role of asymmetry in a simplified setting, in practice, we find that \textit{LayerNorm} also plays an important role by regulating the degree of asymmetry engaged in the GELU activation.

\begin{tcolorbox}[
    enhanced,
    colback=LightGoldenrodYellow!40,
    colframe=black!20,
    arc=2mm, 
    boxrule=0.5pt, 
    boxsep=0mm,
    before skip=5mm,
    after skip=5mm
]
\textbf{Finding 4.1:} \textit{(Interplay between GELU and LayerNorm) LayerNorm activates the inherent asymmetry of GELU by increasing input variance, which in turn causes inter-sequence representation contraction.}
\end{tcolorbox}

In practice, standard transformer MLP blocks typically adopt activation functions from the GLU family, such as GELU~\citep{hendrycks2023gaussianerrorlinearunits, devlin2019bert} and SwiGLU~\citep{shazeer2020glu,bai2023qwentechnicalreport}. This introduces an additional layer of complexity: the interplay between LayerNorm and asymmetric activation. Take GELU as a representative example, $\mathrm{GELU}(x) = x \Phi(x)$, where $\Phi(x)$ is the cumulative distribution function (CDF) of the standard normal distribution. 

As visualized in~\cref{fig:gelu_plot}, although GELU is fundamentally asymmetric, this property becomes effective only when the inputs span a sufficiently wide range of values.
LayerNorm dictates whether this asymmetry is ``engaged'' by regulating the input variance.
With LayerNorm, the inputs exhibit a relatively large standard deviation (e.g., $\sigma \approx 0.55$), forcing them to cover the asymmetric region of GELU. 
In contrast, without LayerNorm, the inputs remain concentrated near the origin (e.g., $\sigma \approx 0.011$), where GELU behaves approximately symmetrically.

\begin{figure}[t]
    \centering
    \includegraphics[width=0.81\textwidth]{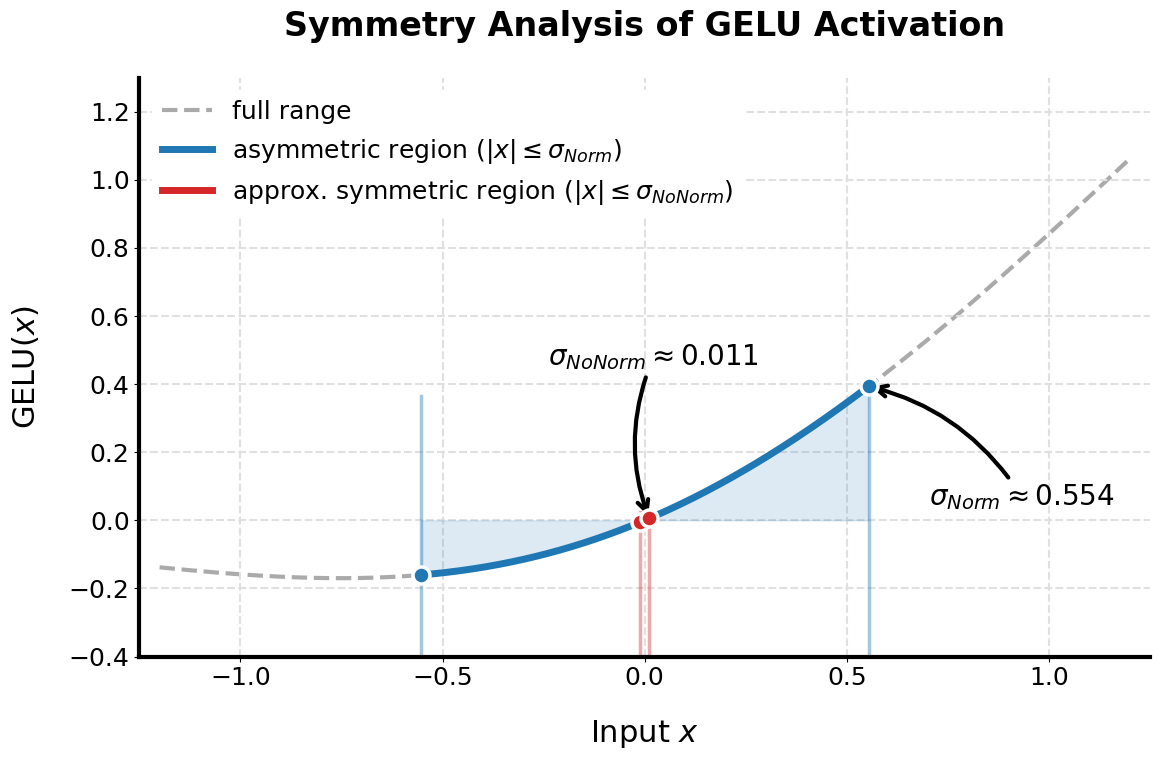}
    \vspace{-10pt}
    \caption{Interplay between GELU and LayerNorm. Inputs with LayerNorm (\textit{Norm}) are pushed into the asymmetric regime of GELU, whereas inputs without LayerNorm (\textit{NoNorm}) stay near zero where GELU is approximately symmetric.}
    \label{fig:gelu_plot}
     \vspace{-10pt}
\end{figure}

\subsection{Interplay between MLP and Self-attention}\label{sec:MLP_attn_interplay}
Although our previous results show that self-attention-only transformers do not inherently cause representation contraction (\cref{fig:MLP_inter_collapse}), the contraction induced by the MLP is substantially amplified by the attention component, thereby strengthening the token-preference bias.
In this section, we analyze the mechanism underlying this synergistic amplification.

First, we extend the experimental setting as in Table~\ref{tab:cosine_sim_results}, to a more fine-grained, layer-by-layer setup to better evaluate the interplay between self-attention and MLP.
We quantify the inter-sequence representation contraction across layers by computing the pairwise cosine similarity between last-token representations of different input sequences. 
Figure~\ref{fig:layerwise_similarity_collapse} illustrates the layer-wise evolution of this similarity across the full, MLP-only, and self-attention-only models.

\begin{figure}[t]
    \centering
    \includegraphics[width=0.8\textwidth]{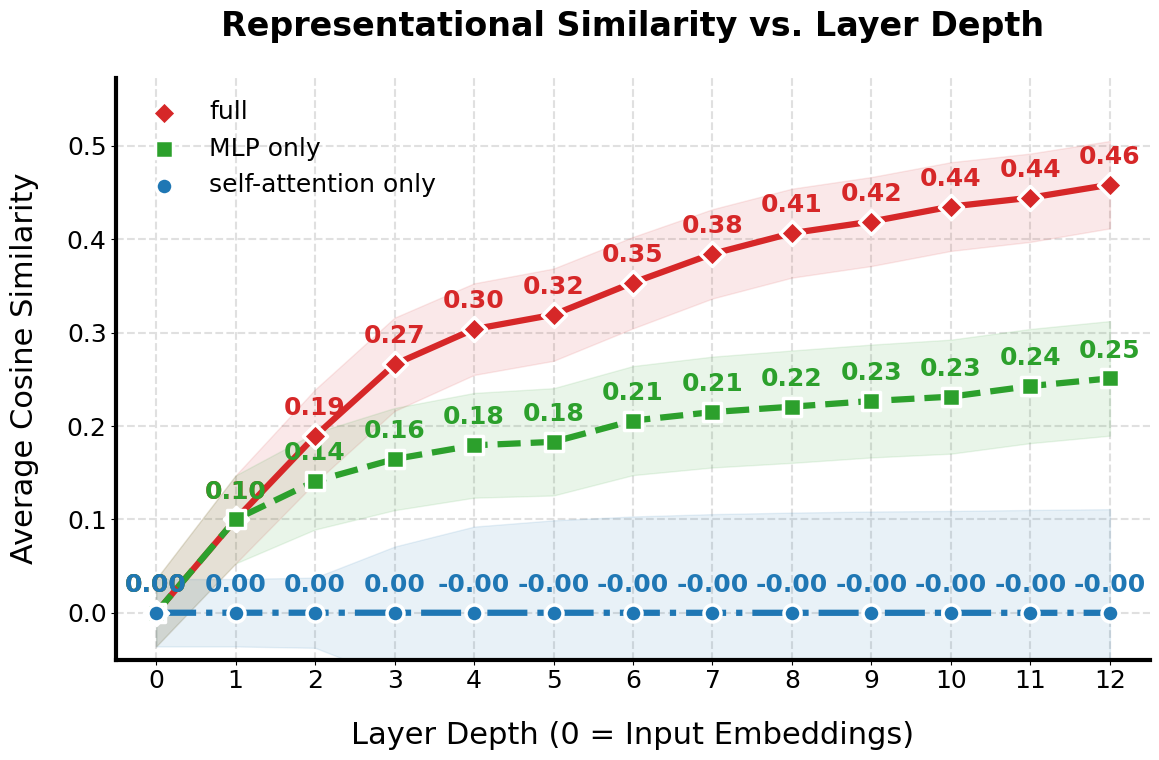}
    \vspace{-8pt}
    \caption{
        Pairwise cosine similarity of the last-token representation across different sequences, measured over the output of each block. 
        The self-attention-only model (blue) remains perfectly orthogonal. 
        The MLP-only model (green) shows a steady increase in similarity. 
        The full GPT model (red) shows a faster and more severe contraction.
    }
     \vspace{-10pt}
    \label{fig:layerwise_similarity_collapse}
\end{figure}

\begin{tcolorbox}[
    enhanced,
    colback=LightGoldenrodYellow!40,
    colframe=black!20,
    arc=2mm, 
    boxrule=0.5pt, 
    boxsep=0mm,
    before skip=5mm,
    after skip=5mm
]
\textbf{Finding 3.2a:} \textit{A self-attention layer can amplify the inter-sequence representation contraction initiated by the MLP, even more so than having an extra MLP layer.}
\end{tcolorbox}

Figure~\ref{fig:layerwise_similarity_collapse} reveals two key divergence points in the layer-wise evolution of representation contraction:
\begin{itemize}
    \item 
    \textbf{Initial divergence (Layer 1)}: A clear separation emerges between the self-attention-only model and the other two configurations. 
    While the full and MLP-only models show an immediate increase in inter-sequence cosine similarity, the self-attention-only model remains nearly flat throughout, confirming that self-attention alone does not initiate contraction.
    \item 
    \textbf{Amplification divergence (Layer 2)}: A second divergence point emerges between the full and MLP-only models. 
    Starting from layer 2, the full model—which alternates between attention and MLP sublayers—exhibits substantially stronger contraction than the MLP-only model.
    This gap widens progressively with depth, supporting the hypothesis that self-attention acts as a powerful amplifier of the contraction effect initially seeded by the MLP blocks.
\end{itemize}

To explicitly test the amplification hypothesis, we isolate the interaction between modules by comparing minimal two-block configurations: an MLP-MLP structure versus an MLP-attention structure. 
For analytical tractability, we retain the simplified ReLU MLP$_0$ block without residual connections and LayerNorm (Definition~\ref{def:relu_MLP}). 
We further introduce a simplified version of self-attention as follows.

\begin{definition}[Attn$_0$ block]
    \label{def:attn0}
Define the Attn$_0$ block as a simplified attention module that only outputs the average of previous input vectors:
\begin{align*}
    \mathrm{Attn}_0(\bm x_1, \cdots, \bm x_T) = \frac{1}{T} \sum_{i=1}^T \bm x_i.
\end{align*}
\end{definition}

This definition distills the standard self-attention mechanism into two core assumptions: \textit{attention allocation} and \textit{value aggregation}. 
Definition \ref{def:attn0} can be thought of as making the following simplification assumptions. 
\begin{itemize}
    \item \textbf{Uniform attention weights:}
    We assume attention weights to be uniform, i.e, attending equally to all previous tokens. 
    This assumption is reasonable under random initialization, since the query-key dot products are centered at zero with minimal variance due to the small initialization scale. 
    As a result, the attention weights after the softmax function tend to be close to uniform. 
    
     \item \textbf{Identity $W_v$ and $W_o$:} 
     We assume the value and output projection matrices are identity matrices.
     This is also reasonable since at initialization, the value matrix $W_v$ and the output projection matrix $W_o$ are both Gaussian random matrices, which do not inherently introduce representation contraction. 
\end{itemize}

Under the simplified attention setting, we compare the following two-block configurations with a shared ``pre-contracted'' one-layer MLP$_0$ layer $\bm{h}= \mathrm{MLP_0}\left( \bm{x} \right)$.
\begin{itemize}
    \item 
    Attn$_0$ $\circ$ MLP$_0(\bm{x})$: $\bm{h}$ is further processed by a simplified Attn$_0$ layer with output 
    $
        \bm{y}_\text{attn} = \mathrm{Attn_0}(\bm{h}).
    $
    \item 
    MLP$_0$ $\circ$ MLP$_0(\bm{x})$: $\bm{h}$ is further processed by a second MLP$_0$ layer with output
    $
        \bm{y}_\text{MLP} = \mathrm{MLP_0}(\bm{h}).
    $
\end{itemize}

In the experimental setup described above, a common first-layer MLP$_0$ is employed to induce an initial, shared representation contraction. 
If self-attention acts as an amplifier of this effect, the Attn$_0$ $\circ$ MLP$_0(\bm{x})$ configuration should exhibit a stronger contraction than the MLP$_0$ $\circ$ MLP$_0(\bm{x})$ counterpart. 
To quantify the contraction, we follow similar procedures as in Table~\ref{tab:cosine_sim_results} and calculate the pairwise cosine similarities between the last-layer last-token
representations across all input sequences.

Results are summarized in Table~\ref{tab:amplifier_results}, which confirms this hypothesized trend. Starting from an initial contraction of 0.31 after the first MLP$_0$ layer, the Attn$_0$ module amplifies the inter-sequence similarity to 0.98, far exceeding the contraction achieved by the MLP-only baseline (0.49).

\begin{table}[t]
\centering
\small
\caption{ 
    Analysis of the amplification effect 
    by comparing the average pairwise cosine similarity of the last-token representations across different sequences and their statistical significance. $\dagger$ The $p$-value is reported as 0.00 due to numerical underflow in floating-point precision.
}
\setlength{\tabcolsep}{45pt}
\label{tab:amplifier_results}
\begin{tabular}{lrr}
\toprule
Experiment setting & Avg. Sim & $p$-value \\
\midrule
$\bm{h}=\mathrm{MLP_0}(\bm{x}) $ & 0.31 & $0.00^{\dagger}$ \\
$\bm{y}_\text{MLP} = \mathrm{MLP_0}(\bm{h})$ & 0.49 & $0.00^{\dagger}$ \\
$\bm{y}_\text{attn} = \mathrm{Attn_0}(\bm{h})$ & \textbf{0.98} & $0.00^{\dagger}$ \\
\bottomrule
\end{tabular}
\end{table}

The simplifications introduced in Definition \ref{def:relu_MLP} and Definition \ref{def:attn0} allow for a rigorous theoretical analysis of of this amplification effect. 
We characterize this relationship in the following proposition:

\begin{proposition}
    [self-attention as a contraction amplifier]
\label{lem:value_aggregation_amplifier}
Let the sequence length be $T$. The expected inter-sequence cosine similarity $\bar\rho$ for the $MLP_0 \circ MLP_0$  case is approximately $0.49$ while that for the $Attn_0 \circ MLP_0$ case is $\frac{T}{T + \pi - 1}$, which converges to 1 as $T$ increases.
\end{proposition}

The proof is provided in Appendix~\ref{sec:proof_prop4}. Proposition~\ref{lem:value_aggregation_amplifier} illustrates that if the inputs are already partially 
contracted (e.g., due to the asymmetric activations in the MLP$_0$), the attention's value aggregation process reinforces the shared directional bias while averaging out the unique, uncorrelated components.
Consequently, the output representations will be
significantly more contracted (i.e., exhibit higher inter-sequence cosine similarity).
This analytical derivation aligns remarkably well with the empirical results in Table~\ref{tab:amplifier_results}: the MLP$_0$ $\circ$ MLP$_0$ configuration yields $\bar\rho\approx 0.49$, while the Attn$_0$ $\circ$ MLP$_0$ configuration gives $\bar\rho=\frac{T}{T+\pi-1}\approx 0.98$ when evaluated at sequence length $T=128$.

\subsection{Attention-Induced Intra-Sequence Representation Contraction}\label{sec:attn_contract}
Unlike MLP, which applies an identical operation to every position, the self-attention mechanism inherently differentiates between positions by aggregating information across the entire sequence.
To understand how this mechanism amplifies global, inter-sequence contraction, we must zoom in on its effect on token representations within a single sequence.
We demonstrate that self-attention causes the hidden representations of distinct tokens to become increasingly similar, a phenomenon we term \textbf{\textit{intra-sequence contraction}}.
This effect provides a mechanistic explanation for the amplification of inter-sequence contraction: when representations across sequences are already partially aligned toward a common direction, seeded by the MLP blocks, self-attention further reinforces this shared direction inside each sequence through local aggregation, thereby amplifying inter-sequence contraction.

To empirically characterize this intra-sequence contraction effect, we sample a single sequence of isotropic Gaussian random vectors ($T=512$ tokens) and feed it through $L=12$ successive standard self-attention blocks. 
After each block, we compute the pairwise cosine similarity among token hidden representations \textit{within the sequence}. 
We repeat this procedure for $N=2000$ independent sequences and report the average and standard deviation.
We further examine how intra-sequence contraction varies as a function of sequence length by evaluating shorter sequences ($T=16$).
The results are shown in Figure~\ref{fig:attn0_vs_real}.

\begin{tcolorbox}[
    enhanced,
    colback=LightGoldenrodYellow!40,
    colframe=black!20,
    arc=2mm, 
    boxrule=0.5pt, 
    boxsep=0mm,
    before skip=5mm,
    after skip=5mm
]
{\textbf{Finding 5:} self-attention mechanism can induce intra-sequence representation contraction, pushing hidden representations of different tokens within a single sequence toward a shared direction. The more self-attention layers, the more severe the contraction.}
\end{tcolorbox}

\begin{figure}[t]
    \centering
    \includegraphics[width=0.80\textwidth]{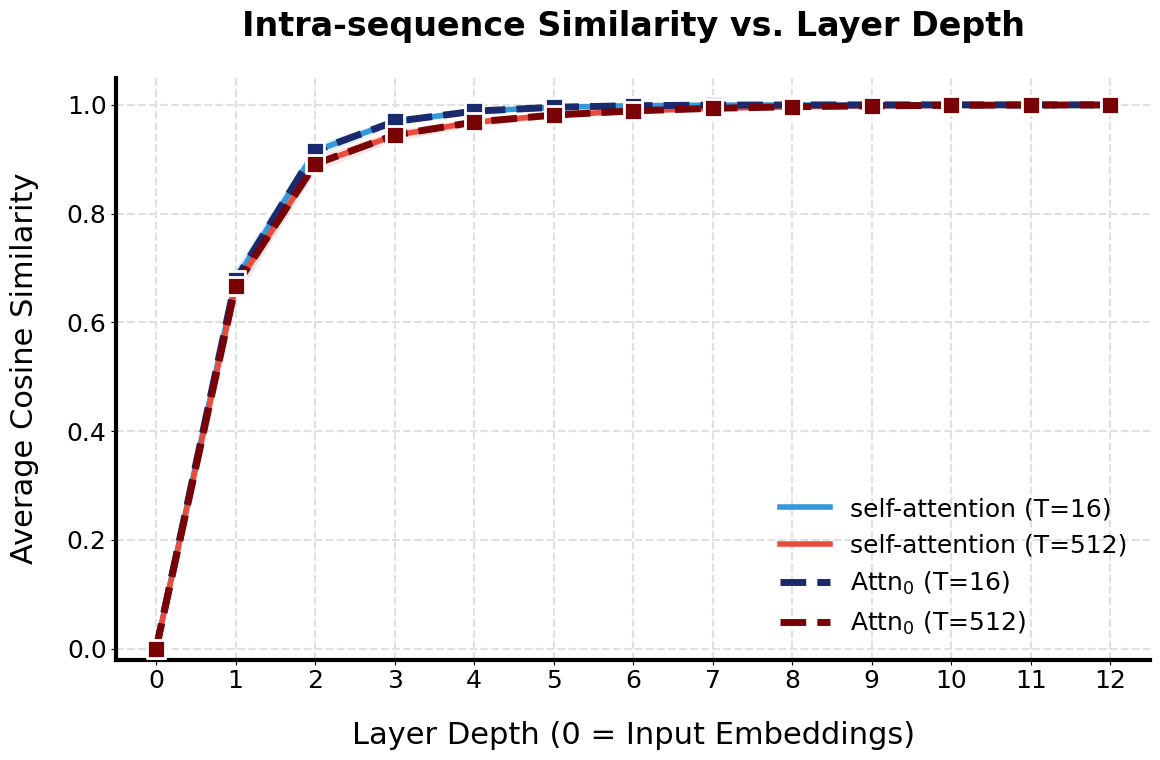} 
    \vspace{-10pt}
    \caption{
        Comparison of intra-sequence cosine similarity between the standard self-attention (solid lines) and the simplified Attn$_0$ block (dashed lines). Red and blue colors correspond to sequence lengths $T=16$ and $T=512$, respectively. The color-shaded regions indicate
        standard deviations of the average intra-sequence cosine similarity computed over 2,000 sequences. 
        The close overlap between the simplified and standard models justifies the use of Attn$_0$ for analyzing the contraction mechanism. 
    }
    \label{fig:attn0_vs_real}
     \vspace{-10pt}
\end{figure}

The solid curves in Figure~\ref{fig:attn0_vs_real} reveal two clear trends. 
First, the intra-sequence cosine similarity increases more rapidly for shorter sequences, indicating that representation contraction is more severe for smaller sequence lengths $T$. 
Second, for both sequence lengths, the similarity grows monotonically with layer depth and converges toward 1. This demonstrates that token hidden representations progressively contract toward a common direction as more self-attention layers are stacked, even when the initial inputs are drawn from an isotropic Gaussian distribution.

For a formal theoretical analysis, we leverage the simplified Attn$_0$ block defined in Definition~\ref{def:attn0}. 
We first empirically verify that this simplified model properly preserves the behavior of standard self-attention. 
As illustrated by the dashed curves in Figure~\ref{fig:attn0_vs_real}, Attn$_0$ trajectories closely match the solid curves corresponding to the full attention module for both evaluated sequence lengths. 
This strong agreement validates our simplification and suggests that, under Gaussian initialization, self-attention effectively behaves as a \textit{uniform aggregation operator}.

Using this validated framework, we next theoretically characterize how sequence length $T$ and layer depth $L$ influence the expected intra-sequence cosine similarity.

\begin{proposition}
\label{lem:intra_contraction}
Consider the simplified Attn$_0$ block setting. Denote $T$ as the sequence length. After $L$ layers of Attn$_0$, the expected intra-sequence cosine similarity $\bar\rho^\prime$ can be characterized by 
\[
\mathbb{E}\left(\bar\rho^\prime(T, L)\right)\approx  1 - \frac{1}{L^2} \left( 1 - \frac{1}{T} \right),
\]
which becomes accurate for $L \ge 3$.
\end{proposition}

The formal proof is provided in~\cref{sec:proof_prop5}. Proposition~\ref{lem:intra_contraction} establishes that for a fixed sequence length $T$, the expected intra-sequence cosine similarity increases monotonically with the number of attention layers $L$ and converges to 1 as $L \to \infty$. 
In other words, self-attention layers will progressively contract token representations within the same sequence toward a common direction. The sequence length $T$ influences the convergence rate via the factor $1 - \frac{1}{T}$: larger $T$ leads to weaker contraction, while smaller $T$ yields a stronger effect. This theoretical result aligns with the empirical trends observed in~\cref{fig:attn0_vs_real}.

Our previous experiments establishes that self-attention drives representations toward a common direction within a single sequence, this same aggregation process inevitably introduces a structural \textit{positional discrepancy}.
Because the mechanism functions analogously to a uniform aggregation operator at initialization, it creates a position-dependent variance decay in the representations.
This positional variance discrepancy substantiates our next finding.

\begin{tcolorbox}[
    enhanced,
    colback=LightGoldenrodYellow!40,
    colframe=black!20,
    arc=2mm, 
    boxrule=0.5pt, 
    boxsep=0mm,
    before skip=5mm,
    after skip=5mm
]
\textbf{Finding 6:} The self-attention mechanism induces a systematic token disparity within each sequence. Specifically, the first token maintains a unique statistical profile, acting as a structural outlier relative to all subsequent positions.
\end{tcolorbox}

\begin{figure}[t]
    \centering
    \includegraphics[width=0.8\textwidth]{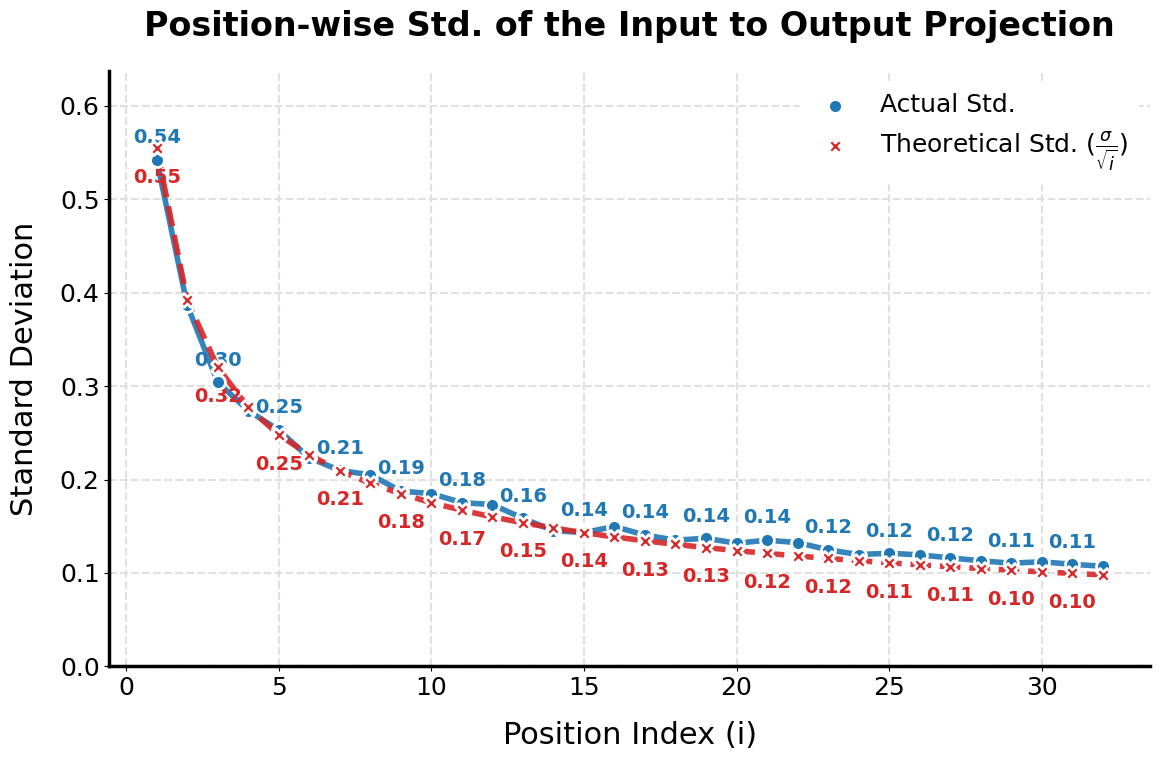}
    \vspace{-10pt}
    \caption{
        Verification of positional variance decay in a randomly initialized self-attention block ($T=32$). The red dashed line represents the empirical standard deviation of the output representation at each token position, while the blue dashed line denotes the theoretical decay curve ($\propto \frac{1}{\sqrt{i}}$) derived from the uniform attention assumption. The near-perfect overlap confirms that the initial token representations degrade in variance as the sequence length increases.
    }
    \label{fig:std_decay}
\end{figure}

To formalize this, we approximate the value vectors $\bm v_j$ at initialization as zero-mean, uncorrelated variables with identical variance $\sigma^2$.
Under the uniform aggregation assumption, the output representation at position $i$, denoted as $\bm{o_i}$, satisfies
\begin{equation} \label{eq:variance_decay}
    \bm{o_i} \approx \frac{1}{i} \sum_{j=1}^{i} \bm{v_j}, \quad \implies \quad \mathrm{Var}(\bm{o_i}) \approx \frac{1}{i^2} \sum_{j=1}^{i} \mathrm{Var}(\bm{v_j}) = \frac{\sigma^2}{i}.
\end{equation}
The variance additivity here stems from the zero cross-position covariance at initialization.
Consequently, the first token ($i=1$) preserves the full initial variance $\sigma^2$, while subsequent tokens ($i > 1$) experience a rapid $1/i$ variance shrinkage.

To validate this phenomenon, we injected Gaussian vectors of length $T=32$ into a randomly initialized self-attention block. We then extracted the output $\bm{o}$ following the value aggregation and calculated the standard deviation of the representation at each position. 
As illustrated in Figure~\ref{fig:std_decay}, the empirical results (red dashed line) nearly perfectly overlap with the theoretical $1/\sqrt{i}$ decay curve (blue dashed line).
 
This inherent peculiarity of the first token is remarkably reminiscent of the \textit{attention sink} phenomenon observed in LLMs~\citep{xiao2023efficient}, where pretrained models assign disproportionately high attention scores to initial tokens (even if semantically irrelevant) to offload probability mass.
Our results suggest that this behavior is not merely learned from data, but is rooted in a fundamental positional bias present at ``birth''. 
This heterogeneity may even effectively compete with explicit positional schemes like RoPE~\citep{su2024roformer}.
In Section~\ref{sec:variance_and_sinks}, we leverage this insight to mitigate the initialization-induced variance discrepancy and establish its causal link to the attention sink.

\section{Persistence of Seed-specific Structural Biases After Training}\label{sec:persistence}

In previous sections, we have investigated randomly initialized transformers and uncovered the underlying mechanism leading to the systematic biases tied to the random initialization. 
A natural question arises: \textit{to what extent do these initial structural properties survive the LLM training process?}
We demonstrate that while a model's specific knowledge evolves through data-driven learning, the underlying structural biases inherited from initialization remain detectable, forming a persistent and idiosyncratic model identity.

\begin{figure}[t]
    \centering
    \includegraphics[width=0.75\textwidth]{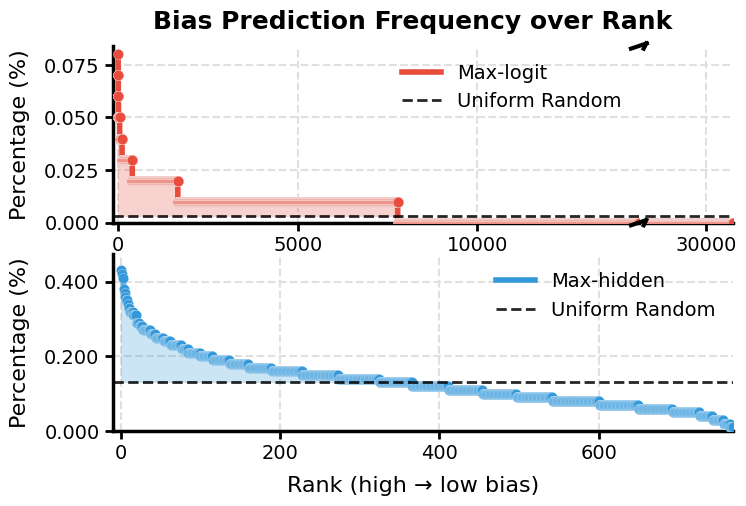}
    \caption{Extreme token preference occurs for both the maximum-logit tokens (top, red) and the maximum-hidden-representation dimensions in the final layer (bottom, blue) across random inputs. The dashed line denotes the expected frequency under a uniform token distribution. The arrows in the top panel indicate a broken x-axis that omits the low-frequency tail ranks.}
    \label{fig:hidden_bias}
    \vspace{-10pt}
\end{figure}

To investigate this persistence with minimal interference from the training objective, we shift our focus from the final next-token predictions to the internal representations of the final hidden layer. 
This one-layer rollback helps mitigate the direct influence of training objectives---since the model is optimized for token prediction rather than hidden representations---and thus provides a cleaner view of the initialization-induced bias. 
This shift in focus is motivated by our earlier findings in \cref{sec:attn_contract}, where we observed that representation contractions also arise in the later hidden layers. We therefore expect that extreme bias patterns should also display in internal representations. In~\cref{fig:hidden_bias}, we plot the frequency with which each token is predicted (upper) and the frequency with which each hidden dimension attains the maximum value at the final layer (lower) over 10,000 random input sequences. The pronounced non-uniform shape in the lower plot indicates that similar extreme bias patterns are also present in the hidden representations of the final layer.
\begin{tcolorbox}[
    enhanced,
    colback=LightGoldenrodYellow!40,
    colframe=black!20,
    arc=2mm, 
    boxrule=0.5pt, 
    boxsep=0mm,
    before skip=5mm,
    after skip=5mm
]
\textbf{Finding 1 (Extended):} \textit{Across diverse random sequences, randomly initialized transformers consistently exhibit an abnormal preference for which hidden dimensions to assign the maximum value in each layer. }
\end{tcolorbox}

We next show that a model trained from a given initialization preserves a \emph{weak but non-negligible correlated bias pattern}. Finding 2.2 reveals an alignment between the model’s predictive preference and the contracted directions reflected in the averaged outputs over random inputs; that is, for some position $j$, the initialized model $f$ exhibits an unusually large expected response $\mathbb{E}_{x \sim \mathcal{D}_{rand}}[f(x)_j]$ compared to a uniform baseline.
We can thus interpret the outputs of a trained model $f'$ through a simple decomposition on an $N\times d_{\text{out}}$ response matrix: for each dimension $j$, 
$$f'(x)_j = b_j(x) + \epsilon_j(x),$$
where $b_j(x)$ represents the initialization-induced bias and $\epsilon_j(x)$ denotes training-specific variation. Therefore, if such bias $b_j(x)$ persists, we would expect to observe, on certain dimensions $j$, that the \textit{output distributions} of two models over random inputs satisfy $\mathrm{Corr}_{x \sim \mathcal{D}_{rand}}\big(f(x)_j,\ f'(x)_j\big) > 0.$

To test this, we focus on the most biased dimensions $\mathcal{M}$, i.e., the top-$m$ ranked dimensions in the contraction direction $\bar{f} \coloneqq \frac{1}{n}\sum_{i=1}^n f(x_i) \in \mathbb{R}^{d_{\text{out}}}$. 
Define 
$$\mathcal{M} = \operatorname*{arg\,max}_{J \subseteq \{1,\dots,d_{\text{out}}\},\, |J|=m} 
\sum_{j \in J} \bar{f}_j,$$
Following the setup of~\cref{tab:integrated_bias_stats}, we generate 10{,}000 random sequences of length 1024. Starting from a randomly initialized model with seed 42, we continuously train this model on the OpenWebText dataset~\citep{Gokaslan2019OpenWeb} for one epoch and evaluate intermediate training checkpoints as target models. For each checkpoint, we compute the correlation between the output distribution at each selected dimension and that of the initialized model. To establish an uncorrelated baseline, we additionally compute the correlation between the initialized model and models trained from a different initialization (seed 123).
We report the mean correlation and standard deviation over the top-$m$ selected dimensions with $m=50$.

\begin{tcolorbox}[
    enhanced,
    colback=LightGoldenrodYellow!40,
    colframe=black!20,
    arc=2mm, 
    boxrule=0.5pt, 
    boxsep=0mm,
    before skip=5mm,
    after skip=5mm
]
\textbf{Finding 7:} \textit{The extreme next-token prediction preference persists during training, which inspires novel LLM fingerprinting methods. }
\end{tcolorbox}

As illustrated in \cref{fig:bias_ckpts}, although the correlation between each checkpoint and the corresponding initialized model is small in absolute value, it is systematically shifted upward relative to the uncorrelated baseline. The two distributions exhibit overlap, but the persistent upward shift supports the presence of initialization-induced bias throughout training. Such persistent bias is therefore exploitable---in~\cref{sec:fingerprint}, we show that these bias profiles can serve as statistically reliable and idiosyncratic model fingerprints.

\begin{figure}[t]
    \centering
        \includegraphics[width=0.78\textwidth]{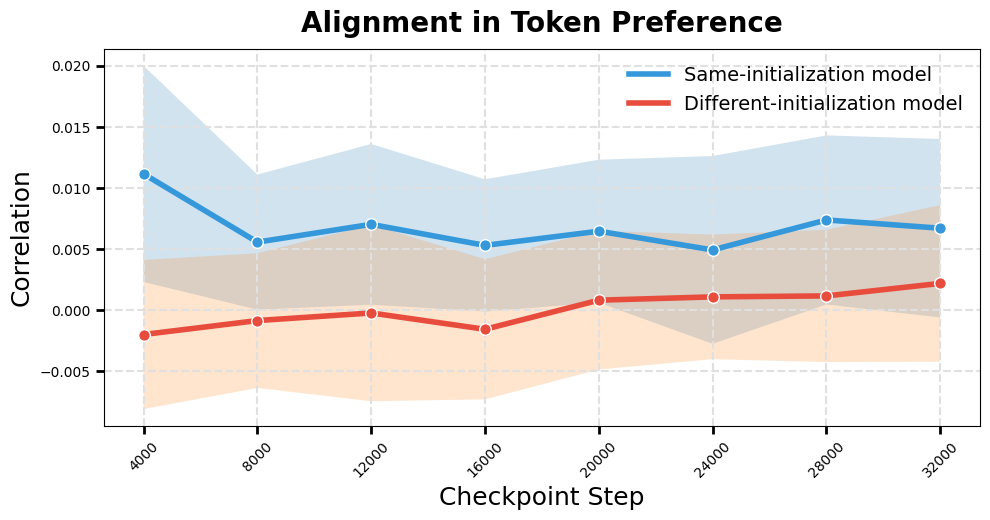}
        \vspace{-5pt}
    \caption{Continuously trained models exhibit weakly correlated bias profiles across training checkpoints, consistently shifted upward relative to the uncorrelated baseline.}
    \label{fig:bias_ckpts}
    \vspace{-10pt}
\end{figure}

\section{Practical Implications}
\label{sec:practical}
\subsection{Seed-Specific Outliers as a Biological-Metaphor LLM Fingerprint}\label{sec:fingerprint}

LLM fingerprinting aims to provide identifiers for trained models, serving as a foundation for model attribution, lineage tracing, and misuse detection.  
However, existing fingerprinting approaches~\citep{pasquini2024llmmap, xu2024instructional, yoon2025intrinsic, zhang2024reef, zeng2024huref, zhang2025matrix, luan2025robust, tsai2025rofl, alhazbi2025llms} primarily rely on semantic behaviors or parameter similarities that emerge only after substantial training, and therefore are usually unavailable at early pretraining stages.
In this section, we show that the observed seed-dependent token preferences can act as {\textit{Biological-Metaphor fingerprints}} for LLMs, \textit{\textbf{SeedPrints}}, that are uniquely determined at models' birth and detectable at \textit{any time} of the subsequent training.

\subsubsection{Algorithm}
\vspace{-5pt}
\cref{fig:bias_ckpts} reveals a key observation: bias profiles between models from the same lineage are weakly correlated, and their correlation scores are distributed above the uncorrelated baselines. However, three challenges remain: (1) the existing bias direction relies on access to the initialization model and its top-$m$ ranked dimensions, which are often unavailable in realistic scenarios; (2) although the same-lineage correlations are distributionally higher, the two distributions overlap, making statistics based solely on mean or standard deviation less powerful; and (3) empirically constructing a reliable uncorrelated baseline involves a non-trivial trade-off between computational cost and estimation accuracy. We next show how our proposed method addresses these three issues.

\paragraph{Resolve (1): Extract identity dimensions between any two models.}
While the top-$m$ ranked dimensions can be directly obtained from the initialized model, such dimensions are unavailable when only two trained models are provided (e.g., an early checkpoint and a late checkpoint). To address this, we consider the \emph{coset} of the two models, with the expectation that if both models inherit bias from the same initialization, their high-preference dimensions will exhibit non-trivial overlap and are more likely to align with a shared bias direction. Conversely, if the two models show little or nearly no overlap, they are likely to be unrelated. We therefore extract dimensions that are jointly prominent in both models and use this intersection as a proxy for the latent initialization-induced \textit{identity dimensions}.

Formally, let $X=\{x_i\}_{i=1}^n$, where each $x_i \in \mathbb{R}^{\ell \times d}$ denotes a random sequence of length $\ell$, obtained by independently sampling $\ell$ random vectors from a $d$-dimensional isotropic Gaussian distribution.
For any model $g$, define the mean output vector $\bar{g} \coloneqq \frac{1}{n}\sum_{i=1}^n g(x_i) \in \mathbb{R}^{d_{\text{out}}},$
where $g$ is either the base model $f$ or the suspect model $f'$, and $d_{\text{out}}$ denotes the output dimensionality (vocabulary size for logits, or hidden size for the final hidden representation).\footnote{Using the final hidden representation instead of logits avoids detokenization noise and is more robust to the rare case that a random sequence appears in the training data.} For each model, we identify its high-preference dimensions as 
$$\mathcal{M}_g = \operatorname*{arg\,max}_{J \subseteq \{1,\dots,d_{\text{out}}\},\, |J|=m} 
\sum_{j \in J} \bar{g}_j.$$
Let $\mathcal{S} \coloneqq \mathcal{M}_f \cap \mathcal{M}_{f'} = \{s_1,\dots,s_{|\mathcal{S}|}\}$
denote the intersection of the two sets. We refer to $\mathcal{S}$ as the \textit{identity dimensions}.

\paragraph{Resolve (2): Hypothesis test on distribution of correlation statistics.} As in~\cref{sec:persistence}, we obtain $|\mathcal{S}|$ empirical Kendall–Tau correlations forming a correlation distribution. We then perform a hypothesis test (e.g., a one-sided $t$-test or the Mann–Whitney $U$-test in~\cref{sec:exp}) to evaluate whether this distribution is significantly higher than a null distribution of no association, constructed from an uncorrelated baseline. The distributional test avoids relying on a fixed summary statistic (e.g., mean or standard deviation), which can be unstable when the two distributions exhibit substantial overlap, and thus provides stronger discrimination in such cases. If the null hypothesis is rejected, we deem $f'$ to be derived from $f$. We declare significance at $p<0.01$; further details are provided in~\cref{alg:main}.

\vspace{-5pt}
\begin{algorithm}[t]
\caption{Distribution Correlation Test on Identity Dimensions}\label{alg:main}
\scalebox{0.85}{
\begin{minipage}{1.15\linewidth}
\begin{algorithmic}[1]
\REQUIRE base model $f$, suspicious model $f'$; random inputs $X=\{x_i\}_{i=1}^n$; fingerprint size $m$; \\
\quad\quad\quad significance level $\alpha$

\leadcomment{$\blacktriangleright$ Step 1: Localize biased dimensions}

\STATE Compute average outputs $\bar f,\ \bar f'$ over $X$
\STATE $\mathcal{M}(f),\ \mathcal{M}(f') \gets$ top-$m$ dimensions of $\bar f,\ \bar f'$
\STATE $\mathcal{S} \gets \mathcal{M}(f)\cap \mathcal{M}(f')$ \hfill (identity dimensions)

\leadcomment{$\blacktriangleright$ Step 2: Form correlation distribution}

\FOR{each $s_j\in \mathcal{S}$}
  \STATE $\tau_j \gets \mathrm{KendallTau}\big(
  \{f(x_i)_{s_j}\}_{i=1}^n,\,
  \{f'(x_i)_{s_j}\}_{i=1}^n
\big)$
\ENDFOR
\STATE $\mathcal{T} \gets \{\tau_j\}_{j=1}^{|T|}$

\leadcomment{$\blacktriangleright$ Step 3: Hypothesis test against null}

\STATE Construct $\mathcal{T}_{\mathrm{null}}$ by applying the same pipeline to two Gaussian matrices $Y^{(1)},Y^{(2)}\sim\mathcal N(0,I)^{N\times d_{\text{out}}}$

\STATE Test $H_0:\mathcal{T}=\mathcal{T}_{\mathrm{null}}$ vs.\ $H_1:\mathcal{T}>\mathcal{T}_{\mathrm{null}}$

\STATE \textbf{Return} {$\mathsf{SameLineage} \gets \mathbf{1}(p\text{-value}<\alpha)$}

\end{algorithmic}
\end{minipage}
}
\end{algorithm}

\paragraph{Resolve (3): Constructing a practical null distribution.}
A final challenge is how to obtain a reliable null baseline without relying on costly empirical comparisons between thousands of independently initialized models. Empirically, the correlation distribution between independently initialized models closely follows a Gaussian centered around zero (see~\cref{appendix:gaussian_null} for more details). Therefore, we construct the null by applying the same dimension-selection and correlation computation pipeline to two independent Gaussian matrices of the same size as the model outputs, as an approximation to unrelated models. This provides a lightweight surrogate for two unrelated models while preserving the dependency structure introduced by selecting identity dimensions.

\subsubsection{Experiments: Birth-to-Lifecycle Biological-Metaphor fingerprinting}\label{sec:exp}

In this section, we demonstrate that our method is a genuine fingerprint: (i) it enables {\textit{birth verification at the seed level}}; and (ii) it remains {\textit{verifiable throughout the full training lifecycle, even for 7B-scale models}}. We test with both the one-sided $t$-test ($t$-test) and Mann–Whitney $U$ test ($U$ test) to demonstrate the robustness to the choice of hypothesis test.

For experiments requiring training from scratch, we use 12-layer, 12-head LLaMA-style models with RoPE~\citep{touvron2023llama, su2021roformer} and Qwen-style models~\citep{team2024qwen2}. We further evaluate fingerprinting in the fine-tuning stage using 7B pretrained variants (see~\cref{tab:fingerprint-models}). Because our random baseline is stochastic, we report $p$-values averaged over 10 independent trials and adopt a significance level of $\alpha=0.01$. Importantly, the absolute magnitude of extremely small $p$-values is not meaningful: once $p$ falls below numerical and sampling noise (e.g., $<10^{-20}$), values like $10^{-260}$ should not be interpreted as stronger evidence than $10^{-20}$—both decisively reject the null. In the main paper, we present results for LLaMA-style models and defer those for Qwen-style models to~\cref{sec:qwen}. The overall conclusions are consistent across the two.

We compare against four state-of-the-art passive fingerprinting methods: Intrinsic Fingerprint~\citep{yoon2025intrinsic}, REEF~\citep{zhang2024reef}, and the two HuRef variants—PCS and ICS~\citep{zeng2024huref}. Additional details are in~\cref{app:more_exp}. Note, in all experiment tables, cell colors indicate lineage: with \colorbox{green!15}{green} denotes models from the same source, and \colorbox{red!15}{red} denotes different sources. 
For example, \colorbox{green!15}{\scalebox{0.8}{$s_{42}^{init}$ vs. $s_{42}^{base}$}} compares a model initialized with seed 42 and its continued-pretrained counterpart, hence green.
By contrast, \colorbox{red!15}{\scalebox{0.8}{$s_{2000}^{init}$ vs. $s_{42}^{base}$}} compares a seed-2000 initialization with a model trained from a seed-42 initialization, hence red.
Additionally, {\color{green!70!black}\checkmark} denotes a correct detection, while ${\color{red!70!black}\times}$ denotes an error.

\input{tables/init_base_compare}

\input{tables/baseline}

\paragraph{Different initialization seeds produce distinct fingerprints.}
Table~\subref{tab:init_model_comparison_alt} reports $p$-values from our correlation tests between pairs of models initialized with different random seeds (42, 123, 1000, and 2000). All $p$-values are consistently $>0.01$, indicating that our method reliably distinguishes models trained from different seeds. This shows that distinct seeds yield distinct fingerprint behaviors, allowing models to be separated ``at birth.''

\paragraph{Training preserves the initialization fingerprint.}
Table~\subref{tab:init_base_comparison} compares each initialization model $s^{init}$ with its descendant $s^{base}$ trained on the OpenWebText dataset~\citep{Gokaslan2019OpenWeb} ($\approx$10B tokens). Across all seed–model pairs, $p$-values are consistently $<0.01$, indicating that their bias profiles remain strongly correlated and thus share a common lineage. In short, the trained model inherits the same fingerprint as its initialization. We also evaluate \textit{{baseline methods}} (results in Appendix~\ref{app:baseline-cross-seed}); without exception, they \textit{{fail to distinguish across seeds}}, which in turn suggests their separability stems from training-induced artifacts rather than properties inherent to a specific model instantiation.

\paragraph{Identical data and order do not make fingerprints converge.} 
In Table~\subref{tab:cross_init_base_comparison}, all four ``suspicious'' models $s^{base}_i$ for $i\in \{42,123,100,2000\}$ are trained on \emph{exactly the same corpus (OpenWebText) and in the same data order} (we fix the training seed to lock the data order); the only difference lies in their initialization seeds $i$. Across all cross-seed pairs, $p$-values remain consistently $>0.01$, in sharp contrast to the near-zero values in Table~\subref{tab:init_base_comparison}. That is, fingerprints remain seed-specific even under identical data and curriculum.

\paragraph{Continual training on diverse datasets does not confound the fingerprint.}
From a copyright perspective, the true weakness of existing fingerprinting methods is their fragility to distribution shifts, which prevents reliable lineage attribution under continual training. In~\cref{tab:fingerprint-continual}, we continue training a base model (seed 1000, pretrained on OpenWebText) on two very different datasets: TinyStories~\citep{eldan2023tinystories} (synthetic children’s stories) and The Stack~\citep{Kocetkov2022TheStack} (permissively licensed GitHub code). We compare (i) true descendants trained from the base, versus (ii) \emph{distractors} derived from a different base model (initialized with seed 123 and trained with a different data order on OpenWebText), then continued training on the \emph{same} corpus. The question is whether attribution methods can identify which descendant truly shares lineage with the base.

We find that \textit{prior baselines all fail under the code setting} (The Stack), misclassifying true descendants as distractors. This indicates that they largely track domain similarity rather than lineage identity: TinyStories is closer in distribution to the pretraining corpus (OpenWebText), while The Stack diverges sharply; such a large distribution shift can easily bypass detection. In contrast, our method correctly attributes lineage across both corpora. Hence, our fingerprint is not a proxy for data distribution: it survives substantial domain shift and persists beyond the initial pretraining stage.

\input{tables/real_model}

\paragraph{From early training to finetuning}
We further compare our method with existing baselines under standard evaluations on pretrained models. In particular, we test suspect models fine-tuned from Llama-2-7b (base model) with data volumes ranging from 5 million to 500 billion tokens. The suspects include diverse downstream variants such as Llama-2-finance-7b~\citep{cxllin2023llama2fin}, Vicuna-1.5-7b~\citep{vicuna2023}, WizardMath-7b~\citep{luo2023wizardmath}, Chinese-LLaMA-2-7b~\citep{chen2023meditron70b}, and Code-Llama-7b~\citep{codellama}. Their fine-tuning data volumes are 5M, 370M, 1.8B, 13B, and 500B tokens, respectively. As shown in \cref{tab:fingerprint-models}, our method consistently maintains $p<0.01$ across all settings.

\subsection{Connection Between Positional Variance Discrepancy and Attention Sinks}
\label{sec:variance_and_sinks}

Our analysis of the self-attention mechanism identifies a fundamental \textit{positional variance disparity} that emerges early in the sequence.
As formalized in Equation~\ref{eq:variance_decay}, the ``running average'' nature of causal attention induces a systematic decay: the variance of a token's representation scales inversely with its position index. 
Consequently, the initial token—exempt from this averaging effect—maintains a significantly higher variance than subsequent tokens, manifesting as a structural and statistical outlier from the onset of training.

This uniqueness of the initial token is a likely precursor to the \textit{attention sink} phenomenon, where the first token disproportionately monopolizes attention scores regardless of semantic context~\citep{gu2024attention}.
This behavior has been characterized as a ``massive activation'' phenomenon \citep{sun2024massive}, wherein initial tokens develop exceptionally large magnitudes to serve as a numerical bias that absorbs residual attention probability. 
Historically, this has been interpreted as a ``null attention'' \citep{vig2019analyzing} or ``no-op'' mechanism \citep{clark2019does}, and its preservation is known to be critical for the stability of sliding-window inference \citep{xiao2023efficient}.

While recent studies have proposed various mechanistic drivers, including spectral subspaces \citep{cancedda2024spectral}, positional ``waiver'' distributions \citep{yan2024unveiling}, and high-norm bands in Query/Key projections \citep{barbero2024round}, these works largely focus on the properties of fully trained models.
In contrast, we hypothesize that the initial positional variance disparity acts as a \textit{statistical anchor} at initialization. This inherent singularity biases the optimization landscape, incentivizing the model to ``latch'' onto the first token as a stable reference point, thereby inducing the emergence of the attention sink.

\subsubsection{Intervention of Attention Sink}
\label{sec:quantitative_verification}

To investigate this causal link, we design an explicit intervention to rectify this inherent variance decay.
Specifically, we apply a \textit{variance calibration} operation to the aggregated representation $\mathbf{o}_i$ (where $\mathbf{o}_i = \sum_{j=1}^i A_{i,j} \mathbf{v}_j$) before it enters the subsequent projection layers. 
We implement two strategies to enforce variance consistency across positions:

\begin{itemize}
    \item \textbf{Positional Amplification:} We amplify the aggregated output by a factor of $\sqrt{i}$ to directly counteract the theoretical $\frac{1}{\sqrt{i}}$ standard deviation decay:
    \begin{equation*}
        \tilde{\mathbf{o}}_i = \sqrt{i} \cdot \mathbf{o}_i = \sqrt{i} \sum_{j=1}^i A_{i,j} \mathbf{v}_j.
    \end{equation*}
    \item \textbf{Positional Attenuation:} To ensure numerical stability while equalizing variance, we propose a normalized variant that scales the output by $\sqrt{\frac{i}{T}}$ (where $T$ denotes the maximum context length):
    \begin{equation*}
        \tilde{\mathbf{o}}_i = \sqrt{\frac{i}{T}} \cdot \mathbf{o}_i = \sqrt{\frac{i}{T}} \sum_{j=1}^i A_{i,j} \mathbf{v}_j.
    \end{equation*}
\end{itemize}
Both methods effectively neutralize the positional bias at initialization, ensuring that representations across all sequence indices contribute with comparable variance to the subsequent layers.

To empirically validate these strategies, we conducted controlled pre-training experiments from scratch, where we adopt the \textit{Nano Llama-2 architecture}~\citep{touvron2023llama} ($d_{model}=768$, 12 layers, 12 heads, 2048 context window, RoPE). 
The models were trained on the {OpenWebText} dataset~\citep{Gokaslan2019OpenWeb} for 200$k$ steps using the AdamW optimizer (full details in Appendix~\ref{app:exp_setup}).
To quantify the attention sink phenomenon, we adopt the threshold-based metric proposed by \citep{gu2024attention}. 
Let $\alpha_{l,h}^1$ denote the importance score of the first token (index 1) in the $h$-th head of the $l$-th layer, calculated as the average attention weight it receives:
\begin{equation*}
    \alpha_{l,h}^1 = \frac{1}{T} \sum_{i=1}^{T} A_{l,h}^{i,1},
\end{equation*}
where a head is classified as a ``sink head'' if $\alpha_{l,h}^1 > \epsilon$. The model-wide \textit{sink rate} is defined as the proportion of such heads:
\begin{equation*}
    \text{Sink}_{\epsilon} = \frac{1}{L \cdot H} \sum_{l=1}^{L} \sum_{h=1}^{H} \mathbb{I}(\alpha_{l,h}^1 > \epsilon).
\end{equation*}
We set $\epsilon = 0.25$ and evaluate this metric on real-world text from the {WikiText-2}~\citep{merity2016pointer} test set, averaging the Sink Rate over 100 randomly selected sequences across lengths ranging from 32 to 512 tokens.

\begin{tcolorbox}[
    enhanced,
    colback=LightGoldenrodYellow!40,
    colframe=black!20,
    arc=2mm, 
    boxrule=0.5pt, 
    boxsep=0mm,
    before skip=5mm,
    after skip=5mm
]
\textbf{Finding 7:} \textit{The initial representation contraction is causally related to the attention-sink appeared after pre-training.}
\end{tcolorbox}  

\begin{figure}[t]
    \centering
    \includegraphics[width=0.80\linewidth]{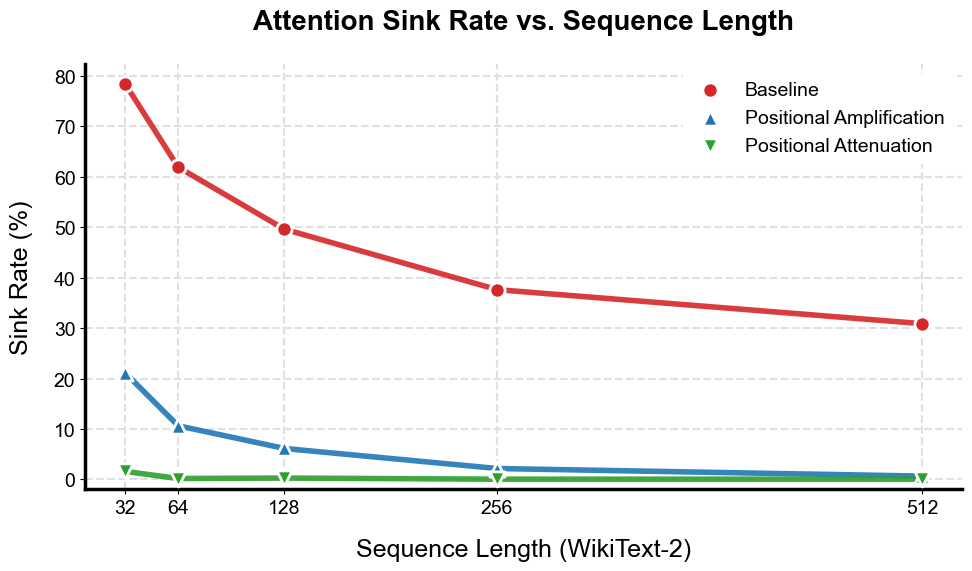}
    \caption{
        \textbf{Quantitative Comparison of Attention Sink Rate on WikiText-2.}
        We report the average Sink Rate across varying sequence lengths (N=100 samples per point). While the Baseline model shows a persistently high sink rate, both Positional Amplification and Attenuation strategies effectively eliminate the phenomenon across all lengths on real text inputs.
    }
    \label{fig:sink_rate_comparison} 
\end{figure}

As illustrated in Figure~\ref{fig:sink_rate_comparison}, the {Baseline} model consistently exhibits a significant sink rate across all sequence lengths, confirming the prevalence of the attention sink phenomenon in the pretraining stage. 
In stark contrast, the model initialized with the \textit{Positional Attenuation} strategy maintains a \textit{near-zero} sink rate regardless of sequence length. 
Meanwhile, the \textit{Positional Amplification} strategy also demonstrates effectiveness, yielding a {significantly lower} sink rate compared to the baseline, although it does not eliminate the phenomenon as completely as the attenuation approach.
This robust quantitative evidence demonstrates that correcting the initial variance disparity effectively mitigates the formation of attention sinks, supporting our hypothesis that they are \textit{learned artifacts} driven by statistical anomalies.

\subsubsection{Impact on LLM Pretraining}
\label{sec:impact_on_ppl}

We further verify that our structural interventions do not compromise the model's pretraining process. 
We compare the validation perplexity (PPL) of the fully converged {Nano Llama-2} models trained with the two standard deviation compensation strategies against the Baseline on two standard benchmarks: {WikiText-2}~\citep{merity2016pointer} and {C4}~\citep{raffel2020exploring}. 
To ensure a rigorous comparison, we conduct the evaluation on {256} sequences, each with a fixed context length of {2048} tokens.

\begin{table}[t]
\small
\centering
\caption{\small
    Analysis of language modeling performance under positional compensation (Nano LLaMA-2), measured by perplexity (PPL; lower is better) on the WikiText-2 and C4 validation sets. Both compensation strategies achieve competitive performance, with Positional Attenuation yielding lower PPL on WikiText-2 and Positional Amplification yielding lower PPL on C4.
}

\setlength{\tabcolsep}{25pt}
\begin{tabular}{l|cc}
    \toprule
    \textbf{Model} & \textbf{WikiText-2 (PPL) $\downarrow$} & \textbf{C4 (PPL) $\downarrow$} \\
    \midrule
    Baseline & 21.33 & 30.36 \\
    Positional Amplification & 21.95 & \textbf{28.70} \\
    Positional Attenuation & \textbf{21.02} & 30.48 \\
    \bottomrule
\end{tabular}
\label{tab:performance_comparison}
\end{table}

As presented in Table~\ref{tab:performance_comparison}, our methods maintain robust performance. 
The \textit{Positional Attenuation} strategy achieves the best performance on WikiText-2 (21.02 vs. 21.33), while the \textit{Positional Amplification} strategy demonstrates superior generalization on the larger C4 dataset (28.70 vs. 30.36).
Overall, both strategies achieve performance comparable to or better than the Baseline, confirming that eliminating attention sinks via variance alignment resolves structural pathologies without sacrificing the model's ability to capture semantic dependencies.

\section{Discussion}\label{sec:discussion}

The investigation presented in this study reveals that transformers are not ``blank slates'' at birth; rather, they are born with an innate structural identity determined by their random initialization seed. 
This identity manifests as extreme next-token preferences. Our mechanistic dissection reveals that this phenomenon is driven by the interaction of asymmetric nonlinear activations in MLP sublayers, which induce inter-sequence representation concentration, and self-attention sublayers, which amplify this effect through intra-sequence aggregation. By simplifying these components into ReLU and uniform attention approximations, we provided a rigorous theoretical framework that accurately predicts the rate of representation contraction observed in our empirical trials.

Beyond providing mechanistic insight, our findings have significant practical implications for model security and architectural design. 
On one hand, we have shown that these initialization-induced biases are persistent throughout the training process. This persistence enables the development of SeedPrint, a novel fingerprinting method that can reliably distinguish models sharing identical architectures and training data but differing only in their initial random seeds.
Unlike existing passive fingerprinting techniques that often track domain similarity or training artifacts, SeedPrint remains robust under substantial distribution shifts and extensive fine-tuning, making it a powerful tool for lineage tracing and copyright auditing.
On the other hand, our analysis identifies a causal link between the structural biases of the initialization regime and the widely observed attention-sink phenomenon, where the value aggregation process of causal attention induces a positional variance decay that makes the first token a statistical outlier. 
By introducing architectural interventions such as Positional Amplification or Attenuation to equalize variance, we demonstrated that attention sinks can be effectively mitigated or eliminated without compromising the model's capture of semantic dependencies.

These discoveries hold the potential for a profound shift in the LLM research paradigm—moving from a focus on ``what the model learns'' to an understanding of ``what the model is born with''. 
For years, the community has treated the transformer as a neutral container for data, yet our work suggests that the architecture's birth state mechanistically roots behaviors previously thought to be learned strategies. 
By demystifying phenomena like attention sinks as architectural byproducts rather than data-driven artifacts, we provide practitioners with a direct pathway to control model stability and identity through principled structural adjustments. This perspective fosters a more transparent and accountable ecosystem for large-scale model development, where model lineage and copyright can be verified from the first epoch of training.

However, this work is not without limitations. Our current analysis primarily focuses on the structural state at initialization and its empirical persistence, without directly modeling the complex dynamics of the training process itself. 
While we have shown that initial biases remain detectable after training, a more rigorous theoretical treatment is needed to understand how these structural tendencies interact with gradient-based optimization and the acquisition of semantic knowledge. 
Furthermore, while we evaluated models up to the 7B parameter scale, the behavior of these biases in ultra-large-scale regimes remains an area for further empirical verification. 
Future research should explore how to intentionally manipulate the initial representational geometry to improve model alignment or enable more efficient pre-training, ultimately demystifying the transition from a structured initialization to a fully capable language model.

\vskip 0.2in

\clearpage
\bibliography{reference}

\clearpage
\appendix
\clearpage
\noindent\textbf{\large Appendix Contents}\par\vspace{15pt}

\startcontents[appendix]
\setcounter{tocdepth}{3}
\printcontents[appendix]{l}{1}{}
\clearpage

\section{Technical Details}
\subsection{Proof of Proposition \ref{lem:asymmetric_collapse}}\label{sec:proof_prop2}

\begin{proof}
We analyze the evolution of the expected cosine similarity through the layers by leveraging the infinite-width limit, where the pre-activation distributions converge to a Gaussian process.

\textbf{The ReLU Recurrence Relation}
Let $\bar{\rho}_{l-1}$ denote the expected cosine similarity between the representations of two inputs at layer $l-1$. In the limit of infinite width, the pre-activations at layer $l$, denoted by $U^{(l)}_1$ and $U^{(l)}_2$, follow a zero-mean bivariate Gaussian distribution with correlation $\bar{\rho}_{l-1}$. For the ReLU activation $\phi(u) = \max(0, u)$, the expected product of the activations (the unnormalized covariance) is given by the arc-cosine kernel of degree 1:
\begin{equation*}
    \mathbb{E}[\phi(U^{(l)}_1)\phi(U^{(l)}_2)] = \frac{\sigma_l^2}{2\pi} \left( \sin \theta_{l-1} + (\pi - \theta_{l-1}) \cos \theta_{l-1} \right),
\end{equation*}
where $\sigma_l^2$ is the pre-activation variance and $\theta_{l-1} = \arccos(\bar\rho_{l-1})$ is the angle between the pre-activation vectors. The variance at the output is obtained by setting $\theta_{l-1}=0$ (perfect correlation):
\begin{equation*}
    \mathbb{E}[\phi(U^{(l)}_1)^2] = \frac{\sigma_l^2}{2\pi} (0 + \pi \cdot 1) = \frac{\sigma_l^2}{2}.
\end{equation*}
The expected cosine similarity at layer $l$ is the ratio of the covariance to the variance:
\begin{equation*}
    \bar\rho_l = \frac{\mathbb{E}[\phi(U^{(l)}_1)\phi(U^{(l)}_2)]}{\mathbb{E}[\phi(U^{(l)}_1)^2]} = \frac{1}{\pi} \left( \sin \theta_{l-1} + (\pi - \theta_{l-1}) \cos \theta_{l-1} \right).
\end{equation*}
By substituting $\sin(\arccos(\rho)) = \sqrt{1 - \rho^2}$ and $\cos(\arccos(\rho)) = \rho$, we obtain the recurrence map $g(\rho)$:
\begin{equation} \label{eq:relu_recurrence}
    \bar\rho_l = \frac{1}{\pi} \left( \sqrt{1 - \bar\rho_{l-1}^2} + (\pi - \arccos(\bar\rho_{l-1})) \bar\rho_{l-1} \right).
\end{equation}

\textbf{Base case ($l=1$):}
For $l=1$, the inputs $\bm X_1$ and $\bm X_2$ are independent isotropic Gaussian vectors, implying $\bar\rho_0 = 0$. Substituting $\bar\rho_0 = 0$ into Equation \eqref{eq:relu_recurrence}:
\begin{equation*}
    \bar\rho_1 = \frac{1}{\pi} \left( \sqrt{1 - 0} + \left(\pi - \frac{\pi}{2}\right) \cdot 0 \right) = \frac{1}{\pi}.
\end{equation*}

\textbf{Monotonicity:} Notice that the function $g(\rho)$ is continuous and differentiable on $[0, 1]$. Since $g(0) > 0$ and $g'(1) = 1$, it can be shown that $g(\rho) > \rho$ for all $\rho \in [0, 1)$. Thus, the similarity is monotonically increasing with $l$.

\textbf{Limiting case ($l\to\infty$):} Since the sequence $\bar{\rho}_l$ is bounded above by 1 and is monotonically increasing, it must converge to a fixed point. We examine $\bar{\rho} = g(\bar{\rho})$. The only stable fixed point in the interval $[0, 1]$ is $\rho^* = 1$. Therefore, $\lim_{l \to \infty} \bar{\rho}_l^{\text{ReLU}} = 1$.

\textbf{The tanh case}
Consider $\phi(u) = \tanh(u)$. Since the inputs $\bm X_1, \bm X_2$ are independent zero-mean Gaussians ($\bar\rho_0 = 0$), the pre-activations $U^{(1)}_1$ and $U^{(1)}_2$ are independent. Let $f_{odd}$ be any odd function. Since the marginal distributions of $U$ are symmetric about zero:
\begin{equation*}
    \mathbb{E}[f_{odd}(U)] = 0 \implies \mathbb{E}[f_{odd}(U_1)f_{odd}(U_2)] = \mathbb{E}[f_{odd}(U_1)]\mathbb{E}[f_{odd}(U_2)] = 0.
\end{equation*}
For $\phi = \tanh$, this implies the output correlation $\bar\rho_1^{tanh} = 0$. By induction, if $\bar\rho_{l-1} = 0$, the pre-activations at layer $l$ remain independent, yielding $\bar\rho_l = 0$ for all $l \ge 1$.
\end{proof}

\subsection{Proof of Proposition \ref{lem:value_aggregation_amplifier}}\label{sec:proof_prop4}

\begin{proof}
We analyze the cosine similarity between the representations of two independent sequences, denoted as $\bm X^{(1)}$ and $\bm X^{(2)}$, passed through the two different architectures. Let the sequence length be $T$. We assume the infinite-width limit where the pre-activations follow a Gaussian Process.

\textbf{Correlation after the First MLP Layer:}
Let $\bm x$ and $\bm x'$ be any two independent input tokens (either from the same sequence or different sequences). From Proposition \ref{lem:asymmetric_collapse}, the expected cosine similarity between their representations after the first ReLU MLP layer, $\bm h = \mathrm{MLP}_0(\bm x)$ and $\bm h' = \mathrm{MLP}_0(\bm x')$, is
$\bar\rho_{1}^{\mathrm{ReLU}} = \frac{1}{\pi}.$

Since the inputs $\bm X^{(1)}_i$ and $\bm X^{(2)}_j$ are all mutually independent standard normal vectors, the hidden representations $\bm h^{(1)}_i$ and $\bm h^{(2)}_j$ share the same pairwise correlation $\rho_1$ for all $i, j$ (and for distinct indices within the same sequence). For the normalized representations, we have:
\begin{equation*}
    \mathbb{E}[\bm h_i^\top \bm h_j] = \bar\rho_{1}^{\mathrm{ReLU}}  = \frac{1}{\pi} \quad \text{for } i \neq j, \quad \text{and} \quad \mathbb{E}[\|\bm h_i\|^2] = 1.
\end{equation*}

\textbf{Case I: MLP$_0$-MLP$_0$}
This architecture applies a second MLP layer element-wise. The inter-sequence similarity is defined as the expected cosine similarity between a token processed by the second layer from sequence 1 and sequence 2. This is equivalent to the propagation of correlation through two ReLU layers starting from zero correlation.
Using the recurrence relation from Proposition \ref{lem:asymmetric_collapse} with $\bar\rho_{1} = 1/\pi$:
\begin{align*}
    \bar\rho_{2}^{ReLU} &= \frac{1}{\pi} \left( \sqrt{1 - \bar\rho_1^2} + (\pi - \arccos(\bar\rho_1))\bar\rho_1 \right) \\
    &= \frac{1}{\pi} \left( \sqrt{1 - \frac{1}{\pi^2}} + \left(\pi - \arccos\left(\frac{1}{\pi}\right)\right)\frac{1}{\pi} \right).
\end{align*}
Substituting values ($\pi \approx 3.14159$, $1/\pi \approx 0.3183$):
\begin{equation*}
    \bar\rho_{2}^{ReLU} \approx \frac{1}{3.1416} \left( 0.948 + (1.895)(0.318) \right) \approx 0.493.
\end{equation*}
This confirms the value is approximately $0.49$.

\textbf{Case II: MLP$_0$-Attn$_0$}
The simplified attention aggregates the values into a single vector (conceptually the sequence embedding). Let $\bm y^{(1)}$ and $\bm y^{(2)}$ be the outputs for the two sequences:
\begin{equation*}
    \bm y^{(1)} = \frac{1}{T} \sum_{i=1}^T \bm h^{(1)}_i, \quad \bm y^{(2)} = \frac{1}{T} \sum_{j=1}^T \bm h^{(2)}_j.
\end{equation*}
We calculate the expected inter-sequence pairwise cosine similarity $\bar\rho = \frac{\mathbb{E}[\bm y^{(1) \top} \bm y^{(2)}]}{\sqrt{\mathbb{E}[\|\bm y^{(1)}\|^2]\mathbb{E}[\|\bm y^{(2)}\|^2]}}$.

\textbf{For the Numerator (cross-correlation):}
Since all pairs across sequences are distinct, $\mathbb{E}[\bm h^{(1)}_i{}^\top \bm h^{(2)}_j] = \bar\rho_1$.
\begin{equation*}
    \mathbb{E}[\bm y^{(1) \top} \bm y^{(2)}] = \frac{1}{T^2} \sum_{i=1}^T \sum_{j=1}^T \mathbb{E}[\bm h^{(1)}_i{}^\top \bm h^{(2)}_j] = \frac{1}{T^2} (T^2 \bar\rho_1) = \bar\rho_1.
\end{equation*}

\textbf{For the denominator (variance):}
For a single sequence, the variance of the mean of $T$ variables with uniform pairwise correlation $\bar\rho_1$ and unit variance is:
\begin{equation*}
    \mathbb{E}[\|\bm y^{(1)}\|^2] = \text{Var}\left(\frac{1}{T}\sum \bm h_i\right) = \frac{1}{T^2} \left( \sum_{i} \text{Var}(\bm h_i) + \sum_{i \neq j} \text{Cov}(\bm h_i, \bm h_j) \right).
\end{equation*}
\begin{equation*}
    \mathbb{E}[\|\bm y^{(1)}\|^2] = \frac{1}{T^2} \left( T \cdot 1 + T(T-1)\bar\rho_1 \right) = \frac{1 + (T-1)\bar\rho_1}{T} = \bar\rho_1 + \frac{1-\bar\rho_1}{T}.
\end{equation*}

\textbf{Resulting Similarity:}
\begin{equation*}
    \bar\rho = \frac{\bar\rho_1}{\bar\rho_1 + \frac{1-\bar\rho_1}{T}} = \frac{T \bar\rho_1}{T \bar\rho_1 + 1 - \bar\rho_1}.
\end{equation*}
Substituting $\bar\rho_1 = 1/\pi$:
\begin{equation*}
    \bar\rho = \frac{T (1/\pi)}{T(1/\pi) + 1 - 1/\pi} = \frac{T/\pi}{(T + \pi - 1)/\pi} = \frac{T}{T + \pi - 1}.
\end{equation*}
As $T \to \infty$, this term approaches 1. This demonstrates that while the MLP layer only moderately correlates the sequences ($\approx 0.49$), the attention mechanism acts as an amplifier, driving the correlation to 1 by suppressing the independent noise components relative to the shared bias induced by the ReLU.
\end{proof}

\subsection{Proof of Proposition \ref{lem:intra_contraction}}\label{sec:proof_prop5}

\begin{proof}
We derive the expression for the mean intra-sequence pairwise cosine similarity $\bar{\rho}^\prime(T, L)$ by decomposing the problem into a finite-size correction term and an asymptotic integral limit.

\textbf{Decomposition into Diagonal and Off-Diagonal Terms}
Let $\bm Z \in \mathbb{R}^{T \times d}$ denote the sequence representation after $L$ layers. The mean intra-sequence pairwise cosine similarity is defined as:
\begin{equation*}
    \bar{\rho}^\prime(T, L) = \frac{1}{T^2} \sum_{i=1}^T \sum_{j=1}^T \mathbb{E}\left[\rho(\bm z_i, \bm z_j)\right].
\end{equation*}
The summation consists of $T$ diagonal terms where $i=j$ (with $\rho=1$) and $T(T-1)$ off-diagonal terms ($i \neq j$). Let $\bar{\rho}^\prime(L) \equiv \lim_{T \to \infty} \bar{\rho}^\prime(T, L)$ denote the expected similarity between any two distinct tokens in the continuous limit. Assuming the sequence is exchangeable or stationary in the limit, the expectation for off-diagonal terms converges to $\bar{\rho}^\prime(L)$. Thus:
\begin{align*}
    \bar{\rho}^\prime(T, L) &= \frac{1}{T^2} \left[ T \cdot 1 + T(T-1) \bar{\rho}^\prime(L) \right] \\
    &= \frac{1}{T} + \frac{T-1}{T} \bar{\rho}^\prime(L) \\
    &= \bar{\rho}^\prime(L) + \frac{1}{T} \left( 1 - \bar{\rho}^\prime(L) \right).
\end{align*}

\textbf{ Derivation of the Infinite Limit $\bar{\rho}^\prime(L)$}
In the limit $T \to \infty$, the discrete cumulative averaging operation at layer $k$, $x^{(k)}_t = \frac{1}{t} \sum_{s=1}^t x^{(k-1)}_s$, converges to the integral operator $Tf(t) = \frac{1}{t} \int_0^t f(s) ds$.
For an input of standard Brownian motion (white noise), applying this operator $N = L-1$ times yields a Gaussian process with covariance kernel $K(s, t)$. For ordered positions $s < t$, the correlation depends only on the ratio $r = s/t$:
\begin{equation*}
    \rho_L(r) = \sqrt{r} \sum_{k=0}^{L-2} C_{N, k} \left( \ln \frac{1}{r} \right)^k,
\end{equation*}
where $C_{N, k}$ are coefficients derived from the Gamma distribution integrals associated with the iterated kernels.
The mean similarity is the expectation over all pairs $(s, t)$. For a uniform distribution on the triangle $0 \le s \le t \le 1$, the probability density function for the ratio $r$ is $f(r) = 2r$.
Integrating the correlation over this domain:
\begin{equation*}
    \bar{\rho}^\prime(L) = \int_0^1 \rho_L(r) \cdot 2r \, dr.
\end{equation*}
Solving this integral using the substitution $u = \ln(1/r)$ transforms the terms into Gamma functions $\Gamma(k)$. The exact closed-form solution is given by:
\begin{equation*}
    \bar{\rho}^\prime(L) = \frac{(L-2)!}{(2L-4)!} \sum_{k=0}^{L-2} \frac{(L-2+k)!}{k!} \left(\frac{2}{3}\right)^{L-1-k}.
\end{equation*}
This formula captures the precise rate at which the representation contracts toward the global mean as the network depth $L$ increases.
\end{proof}

\section{Experimental Details}

\subsection{Initialization Details}
\label{sec:appendix_init_details}
We initialize the embedding layer and all standard linear layers using a normal distribution $\mathcal{N}(0, 0.02)$. However, for the residual projection layers—specifically the attention output projection ($\bm{W}_O$) and the MLP down-projection ($\bm{W}_\text{down}$)—we adopt the scaled initialization scheme from GPT-2~\citep{radford2019language}. For these layers, weights are initialized with a standard deviation of $\frac{0.02}{\sqrt{2 L}}$, where $L$ denotes the total number of transformer layers. This scaling factor mitigates the accumulation of variance along the residual path, thereby ensuring optimization stability in deep transformer architectures.

\subsection{Consistent Next-Token Preference Protocol}
\label{sec:appendix_consistent_details}
To rigorously evaluate intrinsic model biases, all experiments adhere to the following protocol regarding model configuration, input generation, and statistical evaluation.

\paragraph{Model and Weighting Setup}
To ensure the comparability of token ID preferences across different random initializations, we fix the weights of the embedding layer and the LM head across trials for both the \textit{RoPE-enhanced GPT-2 (with weight tying)} and \textit{LLaMA-2 (without weight tying)}. By enforcing identical initialization for these token-interaction layers—regardless of this architectural discrepancy—we isolate the impact of the internal transformer block dynamics on token preference.

\paragraph{Input Sequences}
We evaluate the models using a fixed set of $N$ randomly generated token sequences. By utilizing unstructured random inputs devoid of semantic meaning, we ensure that any observed output preference arises solely from the model's inductive biases (i.e., architecture and initialization) rather than input context.

\paragraph{Evaluation Metrics}
We quantify prediction bias using two primary metrics:
\begin{itemize}
    \item \textbf{Top-1 Token ID:} For every input sequence, we identify the token with the maximum logit ($\arg\max$). The \textit{Top-1 Token ID} is defined as the single token ID that appears most frequently as the predicted next token across the aggregate of all $N$ sequences.
    
    \item \textbf{Statistical Significance ($p$-value):} To validate that the dominance of the most frequent token is not a result of random chance, we perform a hypothesis test with a \textit{Bonferroni correction}. Let $k_{\text{max}}$ be the observed count of the Top-1 token. Under the null hypothesis of a uniform distribution ($p_0 = \frac{1}{V}$), the nominal probability is calculated via the tail of the Binomial distribution:
    \[
        p_{\text{nominal}} = P(X \ge k_{\text{max}}) = \sum_{x=k_{\text{max}}}^{N} \binom{N}{x} p_0^x (1-p_0)^{N-x}, \quad \text{where } X \sim \text{Binomial}(N, p_0).
    \]
    To correct for the multiple comparisons problem inherent in checking a vocabulary of size $V$, the rigorous $p$-value is:
    \[
        p\text{-value} = \min(1.0, \quad V \times p_{\text{nominal}}).
    \]
    A small $p$-value (e.g., $<0.05$) indicates that the observed preference is statistically significant.
\end{itemize}

\subsection{LLM Fingerprinting}

\paragraph{Implementation Details and Licenses}
We train all models with the Hugging Face transformers Trainer~\citep{wolf-etal-2020-transformers}, using Accelerate~\citep{accelerate} for distributed runs. All open-source models are loaded from their official Hugging Face releases and used under their original licenses: Llama models under the Meta Llama Community License, and other models under Apache-2.0. All datasets are downloaded via the Hugging Face Datasets library (the library is Apache-2.0); dataset content follows each dataset’s stated license.

\paragraph{Correlation Metric}
We measure the alignment between two rankings using \textit{Kendall’s Tau ($\tau$)} correlation coefficient. This metric is over the intersection of the top-$m$ biased dimensions from both models to focus on the most salient features.

Let $C$ be the number of concordant pairs and $D$ be the number of discordant pairs. Let $n$ be the total number of tokens in the intersection. Kendall’s $\tau$ is defined as the difference between the proportion of concordant and discordant pairs relative to the total number of pairs:
\[
\tau = \frac{C - D}{\frac{1}{2}n(n-1)}.
\]
A value close to $1$ indicates that the tokens favored by the model's behavior are ordered similarly to those favored by the geometric contraction mechanism.

\subsubsection{Validation of the Gaussian Null Distribution} \label{appendix:gaussian_null}
Our hypothesis test evaluates whether two models share lineage by comparing the distribution of their correlation
statistics over random inputs against a null distribution. The null distribution aims to model the correlation distribution between two \emph{independent models when evaluated on random inputs}.

\paragraph{Empirical Null Distribution from Initialized Models}
The cleanest way to instantiate complete independence is to compare two independently initialized models
(with different random seeds and no training). We estimate this empirical null by evaluating $\sim 2500$ such
pairs and computing Kendall--Tau correlations on 10{,}000 random inputs.
Figure~\ref{fig:init_model_correlation_distribution} shows that the empirical null distribution
closely matches a Gaussian distribution.

\begin{figure}[ht]
    \centering
    \includegraphics[width=0.7\linewidth]{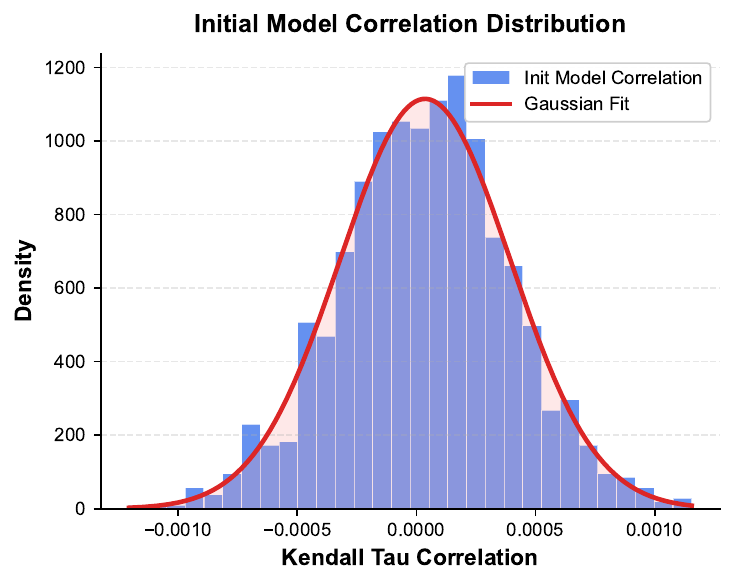}
    \caption{Empirical correlation distribution over $\sim 2500$ pairs of independently initialized models,
    evaluated on random inputs, closely matches a Gaussian distribution centered at zero.}
    \label{fig:init_model_correlation_distribution}
\end{figure}

\paragraph{Constructing the Null Distribution in SeedPrints.}
Approximating the output of neural networks as Gaussian distributions is a widely adopted and effective practice in the auditing literature~\cite{carlini2022membership}. This practice is supported by both empirical evidence and theoretical applications, which show that model output (e.g., logits) often exhibit Gaussian-like behavior across diverse model architectures, data types, and tasks~\cite{lee2019wide}. Therefore, in our work, and based on the previous observation of Gaussian-like Null distirbution (\cref{fig:init_model_correlation_distribution}), we naturally consider Gaussian as a proxy of model output.

Since two random-initialized models have independent parameter matrices, this makes their outputs independent conditional on the same inputs. Therefore,
sampling two independent Gaussian matrices is a natural surrogate for two independently initialized
models under the null. This provides a straightforward way to construct the null:

\begin{enumerate}
    \item Sample two Gaussian matrices of shape $\mathbb{R}^{N \times d_{\text{out}}}$,
    \item Extract most biased dimensions by selecting the top-$m$ ranked average dimensions,
    \item Take the intersection set $\mathcal{S}$ and compute correlations on it.
\end{enumerate}

Intuitively, selecting $\mathcal{S}$ introduces dependencies among dimensions. However, applying the same selection procedure to Gaussian samples preserves this dependency structure under the raw null.

\subsubsection{Baselines Setup}\label{app:more_exp}
We mainly consider four passive fingerprinting baselines (weight- or representation-based).
\textit{{Intrinsic fingerprint}}~\citep{yoon2025intrinsic} (or \emph{PDF} in some papers) compares models via the similarity of the layerwise standard-deviation profiles of attention parameters.
\textit{{REEF}}~\citep{zhang2024reef} computes centered-kernel-alignment (CKA) similarity between feature representations from the same samples across two models.
\textit{{PCS}} and \textit{{ICS}}~\citep{zeng2024huref} (or collectively as \emph{HuRef} in some papers) are weight-similarity methods: PCS flattens all parameters and measures cosine similarity; ICS forms invariant terms from the weights and measures cosine similarity on those invariants. Following~\cite{zhang2024reef}, we use a 0.8 similarity threshold for binary decisions.

\begin{table}[h]
\centering
\setlength{\tabcolsep}{2pt}
\caption{Full results of methods evaluating fingerprints at initialization and after subsequent training.}
\label{app:baseline-cross-seed}
\begin{tabular}{@{}lcccccc@{}}
\toprule
\multirow{2}{*}{\textbf{Model Pair}} 
& \multicolumn{2}{c}{\textbf{Ours}} 
& \multicolumn{4}{c}{\textbf{Baselines}} \\
\cmidrule(lr){2-3} \cmidrule(l){4-7}
 & \textit{$t$-test} & \textit{$u$-test} 
 & \textbf{Intrinsic} & \textbf{REEF} & \textbf{PCS} & \textbf{ICS} \\
\midrule
\cellcolor{green!15}$s_{42}^{init}$ vs. $s_{42}^{base}$     
& 2.20e{-}8$^{\color{green!70!black}\checkmark}$  
& 6.28e{-}8$^{\color{green!70!black}\checkmark}$  
& -0.021$^{\color{red!70!black}\times}$ 
& 0.375$^{\color{red!70!black}\times}$ 
& 0.580$^{\color{red!70!black}\times}$ 
& 0.196$^{\color{red!70!black}\times}$ \\
\cellcolor{green!15}$s_{123}^{init}$ vs. $s_{123}^{base}$   
& 7.09e{-}6$^{\color{green!70!black}\checkmark}$  
& 1.37e{-}5$^{\color{green!70!black}\checkmark}$  
& 0.149$^{\color{red!70!black}\times}$ 
& 0.369$^{\color{red!70!black}\times}$ 
& 0.581$^{\color{red!70!black}\times}$ 
& 0.188$^{\color{red!70!black}\times}$ \\
\cellcolor{green!15}$s_{1000}^{init}$ vs. $s_{1000}^{base}$ 
& 5.58e{-}4$^{\color{green!70!black}\checkmark}$  
& 2.81e{-}3$^{\color{green!70!black}\checkmark}$  
& -0.252$^{\color{red!70!black}\times}$ 
& 0.381$^{\color{red!70!black}\times}$ 
& 0.581$^{\color{red!70!black}\times}$ 
& 0.188$^{\color{red!70!black}\times}$ \\
\cellcolor{green!15}$s_{2000}^{init}$ vs. $s_{2000}^{base}$ 
& 4.00e{-}10$^{\color{green!70!black}\checkmark}$ 
& 1.27e{-}9$^{\color{green!70!black}\checkmark}$ 
& -0.337$^{\color{red!70!black}\times}$ 
& 0.331$^{\color{red!70!black}\times}$ 
& 0.581$^{\color{red!70!black}\times}$ 
& 0.188$^{\color{red!70!black}\times}$ \\
\bottomrule
\end{tabular}
\vspace{-0.3cm}
\end{table}

\subsubsection{Qwen Results}\label{sec:qwen}

\paragraph{Different initialization seeds produce distinct fingerprints}
\cref{tab:qwen_init_model_comparison_alt} presents the $p$-values from correlation tests on Qwen-style models initialized with different random seeds (42, 123, 1000, 2000), evaluated under both the $t$-test and $U$-test. In all cases, the $p$-values remain above $0.01$, confirming that our approach can consistently tell apart models trained from different seeds.

\begin{table}[H]
\vspace{-\baselineskip} 
\centering
\caption{Comparison of fingerprint behaviors between models initialized with different seeds for Qwen-style models.}
\label{tab:qwen_init_model_comparison_alt}
\scalebox{0.95}{
\begin{tabular}{@{}lcc@{}}
\toprule
\multirow{2}{*}{\textbf{Seed Pair}} & \multicolumn{2}{c}{\textbf{Hidden State}} \\
\cmidrule(l){2-3}
 & \textit{$t$-test} & \textit{$U$-test} \\
\midrule
\cellcolor{red!15}$s_{123}$  \,vs.\, $s_{1000}$ & 0.094 & 0.074 \\
\cellcolor{red!15}$s_{1000}$ \,vs.\, $s_{123}$  & 0.125 & 0.094 \\
\cellcolor{red!15}$s_{42}$   \,vs.\, $s_{2000}$ & 0.451 & 0.529 \\
\cellcolor{red!15}$s_{2000}$ \,vs.\, $s_{42}$   & 0.130 & 0.095 \\
\bottomrule
\end{tabular}
}
\vspace{-\baselineskip} 
\end{table}

\paragraph{Training preserves the initialization fingerprint.}
We also compare each initialization model $s^{init}$ with its corresponding trained version $s^{base}$ on the OpenWebText dataset~\citep{Gokaslan2019OpenWeb} ($\approx$10B tokens) for Qwen, as shown in~\cref{tab:qwen_init_base_comparison}. Consistent with the LLaMA-style results, all seed–model pairs yield $p$-values below 0.01. This indicates that training does not erase the initialization fingerprint; instead, the signature is preserved in the descendant model.

\begin{table}[H]
\centering
\caption{Trained models share the same fingerprint behaviors as their initialization models ($p$-value < 0.01).}
\label{tab:qwen_init_base_comparison}
\scalebox{1}{
\setlength{\tabcolsep}{1pt}
\begin{tabular}{@{}lcc@{}}
\toprule
\multirow{2}{*}{\textbf{Model Pair}} & \multicolumn{2}{c}{\textbf{Hidden State}} \\
\cmidrule(l){2-3}
 & \textit{$t$-test} & \textit{$U$-test} \\
\midrule
\cellcolor{green!15}$s_{123}^{init}$  vs. $s_{123}^{base}$   & 7.36e-15 & 3.38e-13 \\
\cellcolor{green!15}$s_{1000}^{init}$ vs. $s_{1000}^{base}$ & 4.41e-13 & 2.01e-11 \\
\cellcolor{green!15}$s_{42}^{init}$   vs. $s_{42}^{base}$   & 1.06e-24 & 2.05e-19 \\
\cellcolor{green!15}$s_{2000}^{init}$ vs. $s_{2000}^{base}$ & 4.87e-24 & 1.92e-20 \\
\bottomrule
\end{tabular}
}
\end{table}

\paragraph{Identical data and order do not make fingerprints converge}
Following the setting in~\cref{tab:cross_init_base_comparison}, we also test whether fingerprint behavior would be erased or confounded by identical data and order for Qwen-style models. \cref{tab:qwen_cross_init_base_comparison} shows that fingerprints remain seed-specific even under identical data and curriculum.

\begin{table}[H]
\centering
\caption{The same dataset and training order do not shape fingerprint behaviors to be identical across different initializations.}
\label{tab:qwen_cross_init_base_comparison}
\scalebox{1.0}{
\setlength{\tabcolsep}{1pt}
\begin{tabular}{@{}lcc@{}}
\toprule
\multirow{2}{*}{\textbf{Model Pair}} & \multicolumn{2}{c}{\textbf{Hidden State}} \\
\cmidrule(l){2-3}
 & \textit{$t$-test} & \textit{$U$-test} \\
\midrule
\cellcolor{red!15}$s_{123}^{init}$  vs. $s_{1000}^{base}$ & 0.286 & 0.254 \\
\cellcolor{red!15}$s_{1000}^{init}$ vs. $s_{123}^{base}$  & 0.036 & 0.040 \\
\cellcolor{red!15}$s_{42}^{init}$   vs. $s_{2000}^{base}$ & 0.026 & 0.043 \\
\cellcolor{red!15}$s_{2000}^{init}$ vs. $s_{42}^{base}$   & 0.112 & 0.123 \\
\bottomrule
\end{tabular}
}
\end{table}

\paragraph{Continual training on diverse datasets does not confound the fingerprint}
\cref{tab:qwen_fingerprint-continual} reports the same test results as~\cref{tab:fingerprint-continual} for Qwen. Our method remains stable and reliable across all cases (even under severe training distribution shifts), whereas many baselines become confused and produce errors.

\begin{table}[H]
\centering
\small
\setlength{\tabcolsep}{1pt}
\renewcommand{\arraystretch}{1.1}
\caption{Fingerprint persistence under continual training on diverse datasets (base model: seed~1000, corpus \texttt{openwebtext}). $U$-test refers to the Mann--Whitney \(U\) test.}
\label{tab:qwen_fingerprint-continual}
\begin{tabular}{lcccccc}
\toprule
\multicolumn{1}{l}{\textbf{Setting}} &
\multicolumn{2}{c}{\textbf{Ours}} &
\multicolumn{4}{c}{\textbf{Baselines}} \\
\cmidrule(lr){1-1}\cmidrule(lr){2-3}\cmidrule(l){4-7}
\textbf{Continual corpus (seed)}
& \textbf{$t$-test} & \textbf{$U$-test}
& \textbf{Intrinsic} & \textbf{REEF} & \textbf{PCS} & \textbf{ICS} \\
\midrule
\cellcolor{green!15}\texttt{TinyStories} (1000) 
& ${8.49e-214}^{\color{green!70!black}\checkmark}$ 
& ${5.09e-71}^{\color{green!70!black}\checkmark}$ 
& ${1.000}^{\color{green!70!black}\checkmark}$ 
& ${0.957}^{\color{green!70!black}\checkmark}$ 
& ${0.999}^{\color{green!70!black}\checkmark}$ 
& ${0.996}^{\color{green!70!black}\checkmark}$ \\
\cellcolor{red!15}\texttt{TinyStories} (123)    
& ${0.256}^{\color{green!70!black}\checkmark}$ 
& ${0.065}^{\color{green!70!black}\checkmark}$ 
& ${0.913}^{\color{red!70!black}\times}$ 
& ${0.199}^{\color{green!70!black}\checkmark}$ 
& ${0.328}^{\color{green!70!black}\checkmark}$ 
& ${0.039}^{\color{green!70!black}\checkmark}$ \\
\cellcolor{green!15}\texttt{the\_stack} (1000)             
& ${1.16e-211}^{\color{green!70!black}\checkmark}$ 
& ${2.30e-76}^{\color{green!70!black}\checkmark}$ 
& ${0.999}^{\color{green!70!black}\checkmark}$ 
& ${0.313}^{\color{red!70!black}\times}$ 
& ${0.995}^{\color{green!70!black}\checkmark}$ 
& ${0.976}^{\color{green!70!black}\checkmark}$ \\
\cellcolor{red!15}\texttt{the\_stack} (123)                
& ${0.610}^{\color{green!70!black}\checkmark}$ 
& ${0.491}^{\color{green!70!black}\checkmark}$ 
& ${0.916}^{\color{red!70!black}\times}$ 
& ${0.255}^{\color{green!70!black}\checkmark}$ 
& ${0.328}^{\color{green!70!black}\checkmark}$ 
& ${0.038}^{\color{green!70!black}\checkmark}$ \\
\bottomrule
\end{tabular}
\vspace{3pt}
\end{table}

\paragraph{All-stage verifiable fingerprints}
Our fingerprinting method on Qwen also demonstrates verifiability at all training stages (\cref{fig:all_stage_verif_qwen}). Across all variants, the suspect model is consistently recognized as belonging to the same lineage, with $p$-values remaining below the 0.01 threshold.

\begin{figure}[H]
    \vspace{-15pt}
    \centering
    \includegraphics[width=0.6\linewidth]{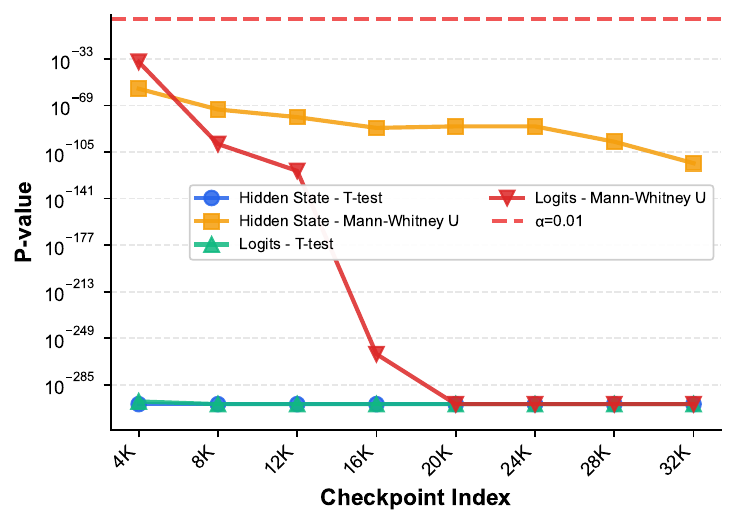}
    \vspace{-3pt}
    \caption{Fingerprint verifies lineage at every checkpoint ($p$-values $< 0.01$) for Qwen structure model.}
    \label{fig:all_stage_verif_qwen}
    \vspace{-10pt}
\end{figure}

\subsection{Experimental Setup for Pre-training}
\label{app:exp_setup}

To facilitate reproducibility, we provide the detailed configurations used for our pre-training experiments.

\subsubsection{Model Architecture}
Our baseline models adopt the standard {Llama 2} architecture~\citep{touvron2023llama}, featuring RMSNorm for pre-normalization, SwiGLU activation functions, and Rotary Positional Embeddings (RoPE). The specific architectural hyperparameters for the models used in our main comparisons are listed in Table~\ref{tab:model_config}.

\begin{table}[h]
    \centering
    \caption{Model Architecture Configurations.}
    \label{tab:model_config}
    \begin{tabular}{lc}
        \toprule
        \textbf{Hyperparameter} & \textbf{Value} \\
        \midrule
        Hidden Size ($d_{\text{model}}$) & 768 \\
        Intermediate Size ($d_{\text{ffn}}$) & 2048 \\
        Number of Layers ($L$) & 12 \\
        Number of Heads ($H$) & 12 \\
        Head Dimension ($d_k$) & 64 \\
        Vocabulary Size & 32,000 \\
        Normalization & RMSNorm \\
        Activation Function & SwiGLU \\
        Position Embedding & RoPE \\
        Context Window & 2048 \\
        \bottomrule
    \end{tabular}
\end{table}

\subsubsection{Initialization and Optimization}
Models are initialized following the standard GPT-2 initialization scheme~\citep{radford2019language}. Please refer to \cref{sec:appendix_init_details} for detailed specifications regarding the standard deviation and residual weight scaling.

We employ the {AdamW} optimizer with $\beta_1=0.9, \beta_2=0.95$. Training follows a cosine learning rate schedule with a linear warmup phase. Detailed optimization hyperparameters are provided in Table~\ref{tab:training_params}.

\subsubsection{Infrastructure}
All experiments were conducted on a single computing node equipped with $2 \times$ NVIDIA A100 (80GB) GPUs, utilizing the Distributed Data Parallel (DDP) strategy. The total training duration was approximately 48 hours.

\subsubsection{Dataset}
The models are pre-trained on the {OpenWebText} dataset~\citep{Gokaslan2019OpenWeb}, an open-source reproduction of the WebText corpus. Data tokenization is performed using the standard Llama 2 tokenizer. Given the target training volume, we train for multiple epochs over the dataset (repeated sampling) to reach the total iteration count.

\begin{table}[h]
    \centering
    \caption{Pre-training Hyperparameters.}
    \label{tab:training_params}
    \begin{tabular}{lc}
        \toprule
        \textbf{Hyperparameter} & \textbf{Value} \\
        \midrule
        Peak Learning Rate & $6.0 \times 10^{-4}$ \\
        Min Learning Rate & $6.0 \times 10^{-5}$ \\
        Warmup Iterations & 2,000 \\
        Max Iterations & 200,000 \\
        Global Batch Size & 128 \\
        Micro Batch Size & 32 \\
        Gradient Accumulation Steps & 2 \\
        Weight Decay & 0.1 \\
        Gradient Clipping & 1.0 \\
        Precision & \texttt{bfloat16} \\
        Optimizer & AdamW \\
        \bottomrule
    \end{tabular}
\end{table}

\end{document}

%% file: tables/init_base_compare.tex

\begin{table}[t]
\vspace{-1em}
\centering
\resizebox{\linewidth}{!}{%
\begin{minipage}{\linewidth}
\centering

\begin{subtable}[t]{0.31\linewidth}
\caption{Comparison of fingerprint behaviors between models initialized with different seeds.}
\vspace{9pt}
\label{tab:init_model_comparison_alt}
\setlength{\tabcolsep}{2pt}
\begin{tabular}{@{}lcc@{}}
\toprule
\textbf{Seed Pair} & \textit{$t$-test} & \textit{$U$-test} \\
\midrule
\cellcolor{red!15}$s_{42}$ vs. $s_{2000}$    & 0.357 & 0.532 \\
\cellcolor{red!15}$s_{123}$ vs. $s_{42}$     & 0.678 & 0.565 \\
\cellcolor{red!15}$s_{1000}$ vs. $s_{123}$   & 0.363 & 0.335 \\
\cellcolor{red!15}$s_{2000}$ vs. $s_{1000}$  & 0.434 & 0.481 \\
\bottomrule
\end{tabular}
\end{subtable}
\hfill
\begin{subtable}[t]{0.33\linewidth}
\caption{Trained models share the same fingerprint behaviors as their initialization ($p$-value < 0.01).}
\label{tab:init_base_comparison}
\setlength{\tabcolsep}{2pt}
\begin{tabular}{@{}lcc@{}}
\toprule
\textbf{Model Pair} & \textit{$t$-test} & \textit{$U$-test} \\
\midrule
\cellcolor{green!15}$s_{42}^{init}$ vs. $s_{42}^{base}$     & 2.20e-8  & 6.28e-8 \\
\cellcolor{green!15}$s_{123}^{init}$ vs. $s_{123}^{base}$   & 7.09e-6  & 1.37e-5 \\
\cellcolor{green!15}$s_{1000}^{init}$ vs. $s_{1000}^{base}$ & 5.58e-4  & 2.81e-3 \\
\cellcolor{green!15}$s_{2000}^{init}$ vs. $s_{2000}^{base}$ & 4.00e-10 & 1.27e-9 \\
\bottomrule
\end{tabular}
\end{subtable}
\hfill
\begin{subtable}[t]{0.31\linewidth}
\caption{The same dataset and training order do not shape fingerprint behaviors to be identical across different initializations.}
\label{tab:cross_init_base_comparison}
\setlength{\tabcolsep}{2pt}
\begin{tabular}{@{}lcc@{}}
\toprule
\textbf{Model Pair} & \textit{$t$-test} & \textit{$U$-test} \\
\midrule
\cellcolor{red!15}$s_{123}^{init}$ vs. $s_{1000}^{base}$  & 0.385 & 0.486 \\
\cellcolor{red!15}$s_{1000}^{init}$ vs. $s_{2000}^{base}$ & 0.035 & 0.096 \\
\cellcolor{red!15}$s_{42}^{init}$ vs. $s_{123}^{base}$    & 0.426 & 0.337 \\
\cellcolor{red!15}$s_{2000}^{init}$ vs. $s_{42}^{base}$   & 0.388 & 0.287 \\
\bottomrule
\end{tabular}
\end{subtable}

\end{minipage}%
}
\end{table}

%% file: tables/baseline.tex
\begin{table}[t]
\centering
\small
\setlength{\tabcolsep}{5pt}
\caption{Fingerprint persistence under continual training on diverse datasets (base model: seed~1000, corpus \texttt{openwebtext}). $U$-test refers to the Mann--Whitney \(U\) test.}
\label{tab:fingerprint-continual}

\begin{tabular}{lcccccc}
\toprule
\multicolumn{1}{l}{\textbf{Setting}} &
\multicolumn{2}{c}{\textbf{Ours}} &
\multicolumn{4}{c}{\textbf{Baselines}} \\
\cmidrule(lr){2-3}\cmidrule(l){4-7}

\textbf{Continual corpus (seed)}
& \textbf{$t$-test} & \textbf{$U$-test}
& \textbf{Intrinsic} & \textbf{REEF} & \textbf{PCS} & \textbf{ICS} \\
\midrule

\cellcolor{green!15}\texttt{TinyStoriesV2\_cleaned} (1000)
& ${0}^{\color{green!70!black}\checkmark}$
& ${\text{7.77e-89}}^{\color{green!70!black}\checkmark}$
& ${1.000}^{\color{green!70!black}\checkmark}$
& ${0.759}^{\color{red!70!black}\times}$
& ${0.999}^{\color{green!70!black}\checkmark}$
& ${0.996}^{\color{green!70!black}\checkmark}$ \\

\cellcolor{red!15}\texttt{TinyStoriesV2\_cleaned} (123)
& ${0.943}^{\color{green!70!black}\checkmark}$
& ${0.902}^{\color{green!70!black}\checkmark}$
& ${0.950}^{\color{red!70!black}\times}$
& ${0.658}^{\color{green!70!black}\checkmark}$
& ${0.332}^{\color{green!70!black}\checkmark}$
& ${0.012}^{\color{green!70!black}\checkmark}$ \\

\cellcolor{green!15}\texttt{the\_stack} (1000)
& ${0}^{\color{green!70!black}\checkmark}$
& ${\text{3.09e-69}}^{\color{green!70!black}\checkmark}$
& ${0.489}^{\color{red!70!black}\times}$
& ${0.557}^{\color{red!70!black}\times}$
& ${0.585}^{\color{red!70!black}\times}$
& ${0.123}^{\color{red!70!black}\times}$ \\

\cellcolor{red!15}\texttt{the\_stack} (123)
& ${0.732}^{\color{green!70!black}\checkmark}$
& ${0.831}^{\color{green!70!black}\checkmark}$
& ${0.445}^{\color{green!70!black}\checkmark}$
& ${0.580}^{\color{green!70!black}\checkmark}$
& ${0.301}^{\color{green!70!black}\checkmark}$
& ${0.026}^{\color{green!70!black}\checkmark}$ \\

\bottomrule
\end{tabular}
\vspace{-5pt}
\end{table}

%% file: tables/real_model.tex


\begin{table}[t]
\centering
\small
\setlength{\tabcolsep}{1pt}
\tabcolsep=1mm
\renewcommand{\arraystretch}{1.15}
\caption{Fingerprinting results \emph{vs.} LLaMA-2-7B. 
Each row compares a target model against LLaMA-2-7B. 
\textbf{\emph{U}-test $p$} reports the $p$-value from our hidden-state correlation test ($<0.01$ indicates a strong signal). 
\textbf{Intrinsic}, \textbf{REEF}, \textbf{PCS}, and \textbf{ICS} report similarity scores (higher = better).}
\label{tab:fingerprint-models}
\scalebox{0.88}{
\begin{tabular}{l r c c c c c}
\toprule
\textbf{Model} & \textbf{\# Tokens} & \textbf{\emph{U}-test $p$ $(< 0.01)$} & \textbf{Intrinsic $\uparrow$} & \textbf{REEF $(\uparrow)$} & \textbf{PCS $(\uparrow)$} & \textbf{ICS $(\uparrow)$} \\
\midrule
Llama-2-finance-7B~\citep{cxllin2023llama2fin}                 & 5M    & \num{1.34e-41}$^{\color{green!70!black}\checkmark}$   & $1.0000^{\color{green!70!black}\checkmark}$ & $0.9950^{\color{green!70!black}\checkmark}$ & $0.9979^{\color{green!70!black}\checkmark}$ & $0.9952^{\color{green!70!black}\checkmark}$ \\
Vicuna-1.5-7B~\citep{vicuna2023}                      & 370M  & \num{1.49e-96}$^{\color{green!70!black}\checkmark}$   & $1.0000^{\color{green!70!black}\checkmark}$ & $0.9985^{\color{green!70!black}\checkmark}$ & $0.9985^{\color{green!70!black}\checkmark}$ & $0.9949^{\color{green!70!black}\checkmark}$ \\
Wizardmath-7B~\citep{luo2023wizardmath}                      & 1.8B  & \num{4.09e-100}$^{\color{green!70!black}\checkmark}$  & $1.0000^{\color{green!70!black}\checkmark}$ & $0.9979^{\color{green!70!black}\checkmark}$ & $1.0000^{\color{green!70!black}\checkmark}$ & $0.9994^{\color{green!70!black}\checkmark}$ \\
Meditron-7B~\citep{chen2023meditron70b}                        & 48B   & \num{5.212e-4}$^{\color{green!70!black}\checkmark}$   & $0.9990^{\color{green!70!black}\checkmark}$ & $0.9978^{\color{green!70!black}\checkmark}$ & $1.0000^{\color{green!70!black}\checkmark}$ & $0.9817^{\color{green!70!black}\checkmark}$ \\
CodeLlama-7B~\citep{codellama}                       & 500B  & \num{2.008e-3}$^{\color{green!70!black}\checkmark}$   & $0.9480^{\color{green!70!black}\checkmark}$ & $0.9947^{\color{green!70!black}\checkmark}$ & $0.6863^{\color{red!70!black}\times}$ & $0.3369^{\color{red!70!black}\times}$ \\
\bottomrule
\end{tabular}
}
\vspace{-5pt}
\end{table}